\documentclass[10pt,reqno]{article}
\usepackage{a4wide}

\usepackage{tikz-cd}

\usepackage{amssymb, amsfonts, amsbsy, latexsym}
\usepackage{amsmath}
\usepackage{amsthm}
\usepackage{stmaryrd}
\usepackage{amssymb}
\usepackage{dsfont}
\usepackage{enumerate}
\usepackage{graphicx}
\usepackage[colorlinks=true, allcolors=blue, hypertexnames=false]{hyperref}
\usepackage[capitalize, nameinlink, noabbrev]{cleveref}
\usepackage{enumerate}
\usepackage{mathrsfs}
\usepackage{tikz}
\usepackage{comment}
\providecommand{\NN}{\mathbb{N}}
\newcommand{\RR}{\ensuremath{\mathbb R}}

\providecommand{\cW}{\mathcal{W}}

\providecommand{\cL}{\mathcal{L}}

\providecommand{\cD}{\mathcal{D}}
\providecommand{\cI}{\mathcal{I}}

\providecommand{\cP}{\mathcal{P}}

\providecommand{\cS}{\mathcal{S}}

\newcommand{\Hilbert}{\mathcal{H}}

\newcommand{\bm}{\ensuremath{\mathbf{m}}}

\providecommand{\proj}{\operatorname{Proj}}

\newcommand\norm[1]{\|#1\|}

\usepackage{amssymb, amsfonts, amsbsy, latexsym}
\usepackage{amsmath}
\usepackage{amsthm}
\usepackage{stmaryrd}
\usepackage{amssymb}
\usepackage{dsfont}
\usepackage{enumerate}
\usepackage{graphicx}
\usepackage[colorlinks=true, allcolors=blue, hypertexnames=false]{hyperref}
\usepackage[capitalize, nameinlink, noabbrev]{cleveref}
\usepackage{enumerate}
\usepackage{mathrsfs}
\usepackage{comment}
\providecommand{\NN}{\mathbb{N}}

\providecommand{\cW}{\mathcal{W}}

\providecommand{\cL}{\mathcal{L}}

\providecommand{\cD}{\mathcal{D}}
\providecommand{\cI}{\mathcal{I}}

\providecommand{\cP}{\mathcal{P}}

\providecommand{\cS}{\mathcal{S}}

\providecommand{\WLrd}{\mathcal{WL}_{r}^{d}}
\providecommand{\WLrdt}{\widetilde{\mathcal{WL}_{r}^{d}}}
\providecommand{\DLrd}{\mathcal{DL}_{r}^{d}}
\providecommand{\DLrdt}{\widetilde{\mathcal{DL}_{r}^{d}}}
\providecommand{\SLrd}{\mathcal{SL}_{r}^{d}}
\providecommand{\SLrdt}{\widetilde{\SLrd}}
\providecommand{\KLr}{\mathcal{KL}_{r}^{1}}
\providecommand{\KLrd}{\mathcal{KL}_{r}^{d}}
\providecommand{\KLrdt}{\widetilde{\KLrd}}

\providecommand{\proj}{\operatorname{Proj}}

\usepackage[utf8]{inputenc} %
\usepackage[T1]{fontenc}    %
\usepackage{hyperref}       %
\usepackage{url}            %
\usepackage{booktabs}       %
\usepackage{amsfonts}       %
\usepackage{nicefrac}       %
\usepackage{microtype}      %
\usepackage{soul}

\usepackage{tcolorbox}
\usepackage{caption}
\usepackage{subcaption}

\usepackage{array}
\usepackage{pifont}

\usepackage{amsmath}
\usepackage{amsfonts}
\usepackage{bm}
\usepackage{parskip}
\usepackage{authblk}
\usepackage{thmtools}
\usepackage{thm-restate}

\usepackage{hyperref}

\usepackage{cleveref}

\providecommand{\cI}{\mathcal{I}}

\usepackage{natbib}

\def\eqref#1{equation~\ref{#1}}

\def\1{\mathbbm{1}}

\DeclareMathAlphabet{\mathsfit}{\encodingdefault}{\sfdefault}{m}{sl}
\SetMathAlphabet{\mathsfit}{bold}{\encodingdefault}{\sfdefault}{bx}{n}

\newcommand{\R}{\mathbb{R}}

\usepackage{wrapfig}
\usepackage[margin=3.3cm]{geometry} 
\usepackage{bbm}
\usepackage{sidecap} %

\usepackage{dsfont}

\usepackage{hyperref}
\usepackage{url}
\usepackage{mathrsfs}
\usepackage{amsthm}
\usepackage{enumerate} 
\usepackage{graphicx}
\usepackage{thm-restate} %

\declaretheorem[name=Theorem,numberwithin=section]{theorem}
\declaretheorem[name=Lemma,sibling=theorem]{lemma}

\declaretheorem[name=Corollary,sibling=theorem]{corollary}

\declaretheorem[name=Proposition,sibling=theorem]{proposition}

\declaretheorem[name=Definition,sibling=theorem]{definition}

\declaretheorem[name=Remark,sibling=theorem]{remark}

\usepackage{bbm}

\newcommand{\X}{\mathcal{X}}
\newcommand{\Y}{\mathcal{Y}}

\newcommand{\f}{\mathbf{f}}

\newcommand{\A}{\mathcal{A}}

\usepackage{tcolorbox}
\newcommand{\ip}[2]{\left\langle#1,#2\right\rangle}
\newcommand{\abs}[1]{\left|#1\right|}
\newcounter{RonCounter}

\newcounter{LeviCounter}

\newcounter{TomCounter}

\title{A Note on  Graphon-Signal Analysis of Graph Neural Networks}
\author{Levi Rauchwerger}
\author{Ron Levie}
\affil{Faculty of Mathematics, Technion -- IIT}
\begin{document}
\maketitle
\begin{abstract}
A recent paper, ``A Graphon-Signal Analysis of Graph Neural Networks'', by Levie, analyzed message passing graph neural networks (MPNNs) by embedding the input space of MPNNs, i.e., attributed graphs (graph-signals), into a space of attributed graphons (graphon-signals). Based on extensions of standard results in graphon analysis to graphon-signals, the paper proved a generalization bound and a sampling lemma for MPNNs. However, there are some missing ingredients in that paper, limiting its applicability in practical settings of graph machine learning. In the current paper, we introduce several refinements and extensions to existing results that address these shortcomings. In detail, 1) we extend the main results in the paper to graphon-signals with multidimensional signals (rather than 1D signals),
2) we extend the Lipschitz continuity to MPNNs with readout with respect to cut distance (rather than MPNNs without readout with respect to cut metric),  3) we improve the generalization bound by utilizing robustness-type generalization bounds, and 4) we extend the analysis to non-symmetric graphons and kernels.
\end{abstract}

\section{Introduction}
In recent years, the surge of interest in machine learning on non-Euclidean domains has brought significant attention to graph neural networks (GNNs), and in particular, to message passing neural networks (MPNNs). These models have demonstrated empirical success \cite{zhou2018graph} across a wide range of domains, from 
computational biology \cite{STOKES2020688,bio02}, molecular chemistry \cite{wang2022graph}, to network analysis \cite{yang2023ptgb}, recommender systems \cite{fan2019graph}, weather forecasting \cite{keisler2022forecasting} and learnable optimization \cite{qian2023exploring, cappart2023combinatorial}.
As a result, there has been substantial interest in establishing solid theoretical foundations, particularly concerning properties such as stability \cite{ruiz2021graph, signal23} and generalization \cite{verma2019stability, yehudai2020local, NEURIPS2020_dab49080, Li2022GeneralizationGO, pmlr-v202-tang23f, signal23, maskey2022generalization, maskey2024generalizationboundsmessagepassing, DIDMs25}. The analysis of stability and generalization often relies on a metric over the input space—that is, a way to quantify similarity between inputs. In the case of graphs, this is typically a pseudometric, since GNNs are generally unable to distinguish between all non-isomorphic graphs.

The recent paper “A Graphon-Signal Analysis of Graph Neural Networks” by Levie \cite{signal23} took a novel approach by embedding the space of attributed graphs, i.e.  graph-signals, into the space of attributed graphons, also called graphon-signals. This embedding allowed the author to leverage tools from graphon theory to analyze message passing neural networks (MPNNs), resulting in generalization bounds and sampling lemmas. Central to this analysis is the introduction of the graphon-signal cut distance \cite{signal23}, a natural extension of the graphon cut distance \cite{lovasz2012large} -- the fundamental notion of distance in the theory of dense graph limits -- to the setting of graph-signals. The space of all graph-signals was shown to be densely embedded in the space of graphon-signals with the graphon-signal cut distance,  and the space of graphon-signals was shown to be compact.

While Levie’s work provides an important and novel framework, \cite{signal23} lacks certain elements which limit  its applicability in practical graph machine learning settings.
In this paper, we introduce refinements and extensions that bridge these gaps.
\begin{itemize}
\item We generalize the cut distance, the MPNN framework, and the main results of \cite{signal23} to graphons and kernels with \emph{multidimensional} signals, allowing both directed and undirected graphons and symmetric or general kernels. Consequently, our analysis covers directed and undirected graphs with weights in either $[0,1]$ or $[-1,1]$.
\item We extend the Lipschitz continuity of MPNNs with respect to the cut metric, to Lipschitz continuity with respect to the cut distance, and show that MPNNs with a readout step are still Lipschitz continuous with respect to the cut distance.
\item We improve the generalization bounds of MPNNs by incorporating robustness-based techniques, achieving a tighter bound, proportional to the square root of the covering number of the space of graphon-signals rather than linear in it.
\end{itemize}

This paper was born out of a practical necessity: a number of researchers, including ourselves, wanted to apply the results of \cite{signal23} to support their constructions and proofs, but the statements in \cite{signal23} were limited in their applicability due to the aforementioned shortcomings. While our extensions of \cite{signal23} are mostly direct, modifying the original proof techniques in a natural way to support the extended claims, we still believe that rigorously verifying these extensions is important. Having a trustworthy peer-reviewed paper with the required extensions will allow future works to freely apply these extended results without  having to verify them each time.

\textbf{We aim for this paper to be cited in conjunction with \cite{signal23} whenever the extended results from the current paper are used.}

\section{Graph-Signals, Graphon-Signals, and Kernel-Signals}\label{sec:background}
In this section, we formulate the spaces of undirected and directed graphon-signals, as well as symmetric and general kernel-signals with multidimensional signals. We remark that most of these definitions have been established in earlier works, such as \cite{lovasz2012large, signal23}, with only minor adaptations by us.

\label{p:set}
For $n \in \NN$, we denote $[n] = \{1, \ldots, n\}$. A \emph{graph-signal} (also known as an \emph{attributed graph}) is a pair $(G, \mathbf{f})$, where $G$ is a weighted or simple, undirected or directed, graph with node set $[n]$, and $\mathbf{f} \in \RR^{n \times k}$ is a node feature matrix. Here, the signal assigns to each node $j \in [n]$ an attribute vector $f_j \in \RR^k$. To avoid confusion, we will always explicitly specify when referring to a directed graph-signal. Any graph can be represented by its adjacency matrix $A \in \RR^{n \times n}$,  
whose entry $A_{ij}$ encodes the weight of the edge from node $i$ to node $j$. \emph{Weighted graphs} are graphs  with weights in $[0,1]$. We refer to graphs whose weights lie in $[-1,1]$ as \emph{$[-1,1]$-weighted graphs}. 

\emph{Graphons} and \emph{kernels} extend the notion of adjacency matrices to nodes in $[0,1]$.\label{def:graphon} The \emph{space of graphons}, denoted $\mathcal{W}_0$, is the set of all Lebesgue measurable symmetric functions $W:[0,1]^2 \rightarrow [0,1]$ (i.e. $W(x,y)=W(y,x)$). \emph{The space of directed graphons}, denoted by $\mathcal{D}_0$, is the set of all Lebesgue measurable functions $W:[0,1]^2 \rightarrow [0,1]$. 
\emph{Kernels} are similar to graphons, but can take negative values. \emph{The space of symmetric kernels} $\cS_1$ is the space of Lebesgue measurable symmetric functions $U:[0,1]^2 \to [-1,1]$, and \emph{the space of kernels} $\cW_1$ is the space of Lebesgue measurable functions $U:[0,1]^2 \to [-1,1]$. Notice that we slightly change the notation for the space of kernels presented in \cite{signal23}. More generally, we define \emph{the space of $\mathbb{R}$-kernels}, denoted by $\cW_{\mathbb{R}}$, as the space of all Lebesgue measurable functions $U:[0,1]^2 \to \mathbb{R}$.

To formalize signals over the node set $[0,1]$, we first introduce the $L^\infty$ space. For $v\in\RR^d$, write $\norm{v}_{\ell^\infty}:=\max_{i\in[d]}|v_i|$. The \emph{$L^\infty$ norm} of a Lebesgue measurable function $f:[0,1]\to\RR^d$ is defined by
\[
\norm{f}_{\infty}
:=\operatorname*{ess\,sup}_{x\in[0,1]}\norm{f(x)}_{\ell^\infty}
=\inf\left\{a\geq0~\middle|~\norm{f(x)}_{\ell^\infty}\leq a
\text{ for almost every }x\in[0,1]\right\}.
\]
Define $\mathcal{L}^\infty([0,1];\RR^d)$ as the space of all Lebesgue measurable functions $f:[0,1]\to\RR^d$ with $\norm{f}_\infty<\infty$.
\label{n:Linfr1} For any $r > 0$, we define the \emph{signal space with range bounded by $r>0$} as
\[
\mathcal{L}^\infty_r([0,1]; \RR^d) := \left\{ f \in \mathcal{L}^\infty([0,1]; \RR^d) ~\middle|~ \norm{f(x)}_\infty \leq r \text{ for almost every } x \in [0,1] \right\}.
\]
Notice that we replace the absolute value used in \cite{signal23} with the $\ell^\infty$ norm for vector-valued signals. Moreover, any function \( f:[0,1]\to\R^d \) can be decomposed into \textit{channels}, with each channel corresponding to a coordinate function \( f_i:[0,1] \to \RR \). Thus, the signal space $\mathcal{L}^\infty_r([0,1]; \RR^d)$ can be identified with the product space \(\prod_{i=1}^d \mathcal{L}^\infty_r([0,1]; \RR),\) which consists of all vectors \( (f_i)_{i \in [d]} \) such that \( f_i \in \mathcal{L}^\infty_r([0,1]; \RR) \) for each \( i \in [d] \).

The spaces of \emph{graphon-signals}, \emph{directed graphon-signals},  \emph{symmetric kernel-signals}, and \emph{kernel-signals}, are defined, respectively, as
\begin{align*}
\WLrd := \mathcal{W}_0 \times &\mathcal{L}^\infty_r([0,1]; \RR^d),\quad \DLrd := \mathcal{D}_0 \times \mathcal{L}^\infty_r([0,1]; \RR^d)\label{n:Linfr2},\\&\quad\SLrd := \mathcal{S}_1 \times \mathcal{L}^\infty_r([0,1]; \RR^d),\quad {\rm and}\quad \KLrd := \mathcal{W}_1 \times \mathcal{L}^\infty_r([0,1]; \RR^d).
\end{align*}
We similarly define the space of $\R$-\emph{kernel-signals} as $\mathcal{RL}_r^d:=\cW_\R\times\mathcal{L}^\infty_r([0,1]; \RR^d)$.

Lastly, any directed or undirected graph-signal can be identified with a corresponding directed or undirected graphon-signal, whereas any directed or undirected $[-1,1]$-weighted graph-signal can be identified as a general or symmetric kernel-signal, respectively. Moreover, any graphon-signal is a kernel-signal, but not vise versa. To formalize this identification, we first introduce partitions. We denote the Lebesgue measure on $[0,1]$ by $\mu$. \label{p:partition} A \emph{partition} \(\mathcal{P}_k = \{P_1, \ldots, P_k\}\) of \([0,1]\) is a set of disjoint measurable subsets of $[0,1]$ whose union is \([0,1]\). A partition is called an \emph{equipartition} if \(\mu(P_i) = \mu(P_j)\) for all \(i,j \in [k]\). The indicator function of a set \(S\) is denoted by \(\mathds{1}_S\). Let \((G, \mathbf{f})\) be a directed or undirected $[0,1]$-weighted or $[-1,1]$-weighted graph-signal with node set \([n]\) and adjacency matrix \(A = \{a_{ij}\}_{i,j \in [n]}\). Consider the equipartition \(\{I_k\}_{k=1}^n\) of \([0,1]\) into intervals \(I_k = [(k-1)/n, k/n)\). The directed or undirected graphon-signal or symmetric or general kernel-signal \((W, f)_{(G,\mathbf{f})} = (W_G, f_{\mathbf{f}})\) induced by \((G, \mathbf{f})\) is defined by
\[
W_G(x,y) = \sum_{i,j=1}^n a_{ij} \, \mathds{1}_{I_i}(x) \, \mathds{1}_{I_j}(y), \quad \text{and} \quad
f_{\mathbf{f}}(z) = \sum_{i=1}^n \f_i \, \mathds{1}_{I_i}(z).\label{def:induced-graphon}
\vspace{-0.4cm}
\]

\section{Message Passing Neural Networks}
Here, we extend the definition of message passing neural networks (MPNNs) with normalized sum aggregation from undirected graph-signals and graphon-signals to cover both directed and undirected graph-signals, as well as directed and undirected graphon-signals, and symmetric and general kernel-signals. The definitions originate from \cite{signal23} and are adapted to incorporate directionality, similarly to the approach of \cite{Rossi23}.%

An MPNN consists of $T$ layers, each comprising two components: a \emph{message passing layer} (MPL) and an \emph{update layer}. In an MPL,  a messages is computed for each edge  by a learnable function of the features of the pair of nodes connected to the edge. Then,  each node aggregates all incoming messages from its neighbors by summing them and dividing by the total number of vertices. The update layer then updates each node’s feature vector by combining the aggregated messages with the node’s previous state via a learnable update function. In our formulation, edge directionality is explicitly incorporated: for a node $i \in V$, we perform separate aggregations over its in-neighbors ($j \to i$) and out-neighbors ($i \to j$).  

For tasks requiring a graph-level output, a \emph{readout layer} aggregates node representations into a single vector,  by averaging all node features. Before presenting the full construction, we define message functions. We denote by $\odot$ the element-wise multiplication along the feature dimension.
\begin{definition}[Message Function]\label{messagefunction}
Let $K ,d,p\in\NN$. A sequence of Lipschitz continuous functions  $\{\xi_{k,\rm rec} , \xi _{k,\rm trans}: \mathbb{R}^d \mapsto \mathbb{R}^p\}_{k\in[K]}$ is called a sequence of \emph{Receiver} and \emph{transmitter message functions}  respectively. The \emph{message function} 
 corresponding to the receiver and transmitter sequence is the function   $\Phi:\mathbb{R}^{2d} \mapsto \mathbb{R}^p$ defined by 
 $$\Phi(a, b) = \sum^K_{k=1} \xi_{k,\rm rec}(a) \odot\xi_{k,\rm trans}(b).$$ 
\end{definition}
The above construction is quite general. It is a standard result that broad classes of functions 
$F : \R^d \times \R^d \to \R^C$ (for instance, $L^2$ functions) can be approximated by finite linear combinations of simple tensors, i.e., 
\(
F(a,b) \approx \sum_{k=1}^K \xi_1^k(a)\odot\xi_2^k(b).
\)
Consequently, square-integrable message functions can be approximated in $L^2$ by finite sums of the above form with Lipschitz factors. Examples of related graph neural network architectures can be found in \cite{comptreeImportant1, xu2018powerful, defferrard2017convolutional, kipf2017semirefsupervised, levie2018cayleynets}.

We now define message passing neural networks (MPNNs) as a pair of sequences, one consisting of update functions and the other of message functions.

\begin{restatable}[Message Passing Neural Network]{definition}{defmppnmodel}~\label{definition:MPNNmodel}
Let $T\in\NN$ and $d_0,\ldots,d_T,p_0, \ldots, p_{T-1},d\in\NN$. We call the tuple $(\mu, \Phi)$ such that $\mu$ is any sequence $\mu=(\mu^{(t)})_{t=1}^{T}$ of Borel measurable functions $\mu^{(t)}:\R^{d_{t-1}+p_{t-1}}\mapsto\R^{d_{t}}$ that are bounded on bounded sets, for $1\leq t\leq T$, and $\Phi$ is any sequence $\Phi=(\Phi^{(t)})_{t=1}^{T}$ of message functions $\Phi^{(t)}:\R^{2d_{t-1}}\mapsto\R^{p_{t-1}}$, for $1\leq t\leq T$, a \emph{$T$-layer MPNN}, and call $\mu^{(t)}$  \emph{update functions}. For $\psi: \RR^{d_{T}}\mapsto \RR^d$, we call the tuple $(\mu, \Phi, \psi)$ an \emph{MPNN with readout}, where $\psi$ is called a \emph{readout function}. We call $T$ the \emph{depth} of the MPNN, $d_0$ the \emph{input node feature dimension}, $d_1,\ldots, d_T$ the \emph{hidden node feature dimensions}, $p_0,\ldots, p_{T-1}$ the \emph{hidden edge feature dimensions}, and $d$ the \emph{output feature dimension}.
\end{restatable}

An MPNN model processes graph-signals as a function as follows.

\begin{definition}[MPNNs on Graph-Signals]\label{definition:altGraphFeat}
Let $(\mu,\Phi,\psi)$ be a $T$-layer MPNN with readout. Let $(G,\f)$ be a graph-signal (directed, undirected, unweighted, weighted, or $[-1,1]$-weighted), with node set \([n]\) and adjacency matrix \(A = \{a_{ij}\}_{i,j \in [n]}\), where $\f:[n] \mapsto \R^{d_0}$. The \emph{application} of the MPNN on $(G,\f)$ is defined as follows. For each $i\in[n]$, define $\mathbf{f}_{i}^{(0)} := \f_i$. Define the hidden node representations $\mathbf{f}_i^{(t)}:=\Theta_t(G,\mathbf{f})_i\in \mathbb{R}^{d_t}$ at each layer $t\in[T]$ and the graph-level output $\Theta(G,\mathbf{f}) \in \mathbb{R}^d$  respectively by
\begin{align*}
    \Theta_t(G,\mathbf{f})_i &:= \mu^{(t)} \Big(\mathbf{f}_i^{(t-1)} , \frac{1}{n} \sum_{j\in %
    [n]}a_{ij}\Phi^{(t)}(\mathbf{f}^{(t-1)}_i, \mathbf{f}^{(t-1)}_j)%
    \Big)\quad
    \text{and}\quad\Theta(G,\mathbf{f}):= \psi \Big(\frac{1}{n} \sum_{j\in [n]} \mathbf{f}^{(T)}_j\Big).
\end{align*}
\end{definition}
To extend the MPNN framework to graphons and kernels, we transition from a discrete set of nodes to a continuous domain, modeled as a general atomless standard probability space (note that $[0,1]$ is a canonical example of such a space), by replacing the normalized sum with an integral.
\begin{definition}[MPNNs on Graphon-Signals and kernel-signals]\label{definition:altGraphonFeat}
Let $(\mu,\Phi,\psi)$ be a $T$-layer MPNN with readout, and $(W,f)$ be a (directed or undirected) graphon-signal or a (symmetric or general) kernel-signal where $f:[0,1] \mapsto \R^{d_0}$. The \emph{application} of the MPNN on $(W,f)$ is defined as follows: for almost every $x\in[0,1]$, initialize $f_{x}^{(0)} := f_x$, and compute the hidden node representations $f_x^{(t)}:=\Theta_t(W,f)_x\in\mathbb{R}^{d_t}$ at layer $t$, with $1\leq t \leq T$ and the graphon-level output $\Theta(W,f) \in \mathbb{R}^d$  by
\begin{align*}
    \Theta_t(W,f)_x &:= \mu^{(t)} \Big(f_x^{(t-1)} , \int%
    W(x,y) \Phi^{(t)}(f^{(t-1)}_x, f^{(t-1)}_y) dy%
    \Big) 
    \quad\text{and}\quad
    \Theta(W,f):= \psi \Big( \int%
    f^{(T)}_x \text{d}x 
    \Big).
\end{align*}
\end{definition}

\label{p:mpnnt}When the MPNN has no readout function, we use $\Theta(W,f)$ and $\Theta_T(W,f)$ interchangeably.  
\cite[Lemma~E.1]{signal23} shows that, under the above definitions, 
applying an MPNN to a graph-signal and then inducing a graphon-signal yields the same representation as first inducing a graphon-signal and then applying the MPNN. Moreover, the same commutativity property holds for readout: applying readout directly to the graph-signal, or first inducing the graphon-signal and then applying readout, gives the same 
result. In other words, the following diagram commutes:
\[
\begin{tikzcd}[column sep=1.6em, row sep=2.6em]
(G,\f) \arrow[rr, "\text{induction}"] \arrow[d, "\Theta_T"'] & & (W_G,f_{\f}) \arrow[d, "\Theta_T"] \\
(G,\Theta_T(G,\f)) \arrow[rr, "\text{induction}"] \arrow[dr, "\psi\left(\frac{1}{n}\sum_{j}\,\cdot\,\right)"'] & & (W_G,\Theta_T(W_G,f_{\f})) \arrow[dl, "\psi\left(\int\,\cdot\;dx\right)"] \\
& r\in\RR^d &
\end{tikzcd}
\]
In \cref{Ap:Graphon-signal MPNNs}, we extend this result to directed and undirected graphon-signals, as well as to symmetric and general kernel-signals, with higher-dimensional signals.
\section{Extended Graphon-Signal Analysis}
In this section, we extend the graphon-signal analysis present in \cite[Section 3]{signal23} for graphons with 1D signals to graphons and kernels with multidimensional signals. All proofs are
given in the appendix. We start with definitions.

\subsection{The Cut Distance and Quotient Spaces}\label{section:cutnorm}

The \emph{cut norm} was first introduced by \cite{frieze1999quick} and was later extended by \cite{lovasz2012large} to \emph{cut distance}, which provides the central notion of convergence in the theory of dense graph limits \cite{lovasz2012large}. \cite{signal23} generalized these notions to graphon equipped with one-dimensional signals. Here, we extend the graphon-signal cut distance by providing two natural extensions of the signal cut norm to multidimensional signals. %
\begin{definition}[The Cut Norm]
For a bounded Lebesgue measurable kernel $W:[0,1]^2\to\RR$, define its \emph{cut norm} by
\[
\|W\|_{\square}:=\sup_{A,B\subseteq[0,1]}
\left|\int_{A\times B}W(x,y)\,dx\,dy\right|,
\]
where the supremum is taken over measurable sets $A,B$.
\end{definition}
Two definitions arise as natural extensions of the one-dimensional signal cut norm \cite[Definition 3.1]{signal23}.

\begin{definition}[The Cut Norm of a Signal]\label{definition:multidimCutNorm}
The \emph{cut norm} of $f\in\mathcal{L}^1([0,1];\R^d)$ is defined to be
\[
\|f\|_\square := \frac{1}{d}\sup_{S \subset [0, 1]} \norm{\int_S f(x)\,dx}_{1},
\label{eq:3}
\]
where the supremum is taken over the measurable subsets $S\subseteq [0,1]$. 
\end{definition}

The following definition has previously been proposed as the standard signal cut norm by \cite{10.5555/3737916.3739744}.

\begin{definition}[The Product Cut Norm of a Signal]\label{definition:productNorm}
The \emph{product cut norm} of $f\in\mathcal{L}^1([0,1];\R^d)$ is defined as
\[
\|f\|_{\square_\times} :=\frac{1}{d}\norm{ \left( \norm{f_i}_{\square}\right)_{i\in[d]}}_{1},
\]
 where $f(x)=(f_i(x))_{i\in[d]}$, i.e., $\{f_i\}_{i\in[d]}$ are the channels of $f$, and $\norm{f_i}_{\square}$ is defined via Definition \ref{definition:multidimCutNorm} with $d=1$.
\end{definition}

For any $r>0$ and any $\RR$-kernel-signal $(W,f)\in\mathcal{RL}_r^d$ with $W$ bounded, the \emph{cut norm} is defined as
\[
\|(W, f)\|_\square := \|W\|_\square + \|f\|_\square,\label{eq:4}
\] 
whereas the \emph{product cut norm} of an $\RR$-kernel-signal $(W,f) \in \WLrd$ is defined as
\[
\|(W, f)\|_{\square_{\times}} := \|W\|_\square + \|f\|_{\square_\times}.
\]
The signal $L^1$ norm is defined as $\norm{f}_1 :=\int_{0}^1 \norm{f(x)}_1 dx$ for $f\in\mathcal{L}^1([0,1];\R^d)$, when the $1$-norm inside the integral is the $\ell^1$ norm on $\R^d$, i.e., $\|x\|_1 = \sum_{i=1}^{d} |x_i|$, for $x=(x_i)^d_{i=1}\in\R^d$. The following claim shows that the signal cut norm, the signal cut product norm, and the signal $L^1$-norm are all equivalent. %

\begin{restatable}[]{claim}{normequivalency}~\label{claim:cutProp}
Let $f\in\mathcal{L}^1([0,1];\R^d)$. Then
\begin{equation*}
\frac{1}{2d^2}\norm{f}_1
\leq\frac{1}{d}\norm{f}_{\square_\times}
\leq\norm{f}_{\square}
\leq\norm{f}_{\square_\times}
\leq\frac{1}{d}\norm{f}_1.
\end{equation*}
\end{restatable}

The proof of \cref{claim:cutProp} is given in \cref{AP:Notation}. \cref{claim:cutProp} extends \cite[Equation (9)]{signal23}, which presents equivalence between the $L^1$-norm and the signal cut norm in the single dimensional case. 

The pseudometric induced by the graphon-signal cut norm is called the \emph{cut metric}. It is a metric after identifying pairs whose kernel and signal components agree almost everywhere.

\begin{definition}[The Cut Metric]
The \emph{cut metric} between any two kernel-signals $(W,f),(V,g)\in \KLrd$ is defined to be $$d_{\square}((W,f),(V,g)) := \norm{(W,f)-(V,g)}_{\square}.$$
\end{definition}
Now we follow the construction of \cite{signal23} to define the cut distance for kernel--signal with multidimensional signals. 
Let $S'_{[0,1]}$ denote the collection of measurable bijections between co-null subsets of $[0,1]$. Explicitly,\label{p:nullset}  
\[
S'_{[0,1]}:=\left\{\phi:A\to B~\middle|~
\begin{array}{l}
A,B\text{ are co-null subsets of }[0,1],\\
\text{for every Lebesgue measurable }S\subseteq A,\\
\phi(S)\text{ is Lebesgue measurable and }\mu(S)=\mu(\phi(S))
\end{array}
\right\},
\]
where $\mu$ denotes Lebesgue measure. For $\phi \in S'_{[0,1]}$, a kernel $W \in \cW_1$, and a signal $f \in \mathcal{L}_r^{\infty}([0,1];\R^d)$, we set  
\[
W^{\phi}(x,y) := W(\phi(x),\phi(y)), 
\qquad 
f^{\phi}(z) := f(\phi(z)),
\]  
and write $(W,f)^{\phi} := (W^{\phi},f^{\phi})$. 
Both $W^{\phi}$ and $f^{\phi}$ are defined only up to a null set; on such sets we fix their values (and those of $W$ and $f$) to be $0$. 
This convention is harmless, since the cut norm is insensitive to modifications on null sets.  

\begin{definition}[The Cut Distance]The \emph{cut distance}\label{eq:gs-metric} between two kernel-signals $(W,f)$ and $(V,g)$ is defined to be 
    \[\delta_{\square}((W,f),(V,g)): = \inf_{\phi\in S'_{[0,1]}} d_{\square}((W,f),(V,g)^\phi).\] 
\end{definition}
Similarly to the cut distance for graphon-signals with single dimensional signals (see \cite[Section 3.1]{signal23}), the kernel-signal distance $\delta_{\square}$ (\cref{eq:gs-metric}) is a pseudo-metric. We introduce an equivalence relation $(W,f)\sim (V,g)$ if $\delta_{\square}((W,f),(V,g))=0$, through which we define the quotient spaces $$\widetilde{\WLrd}:=\WLrd/\sim,\quad\widetilde{\DLrd}:=\DLrd/\sim,\quad\SLrdt:=\SLrd/\sim,\quad\widetilde{\KLrd}:=\KLrd/\sim,$$\label{p:graphoncutspace} and define $\delta_{\square}([(W,f)],[(V,g)]):=\delta_{\square}((W,f),(V,g))$ where $[(W,f)],[(V,g)]$, are the equivalence classes of
$(W,f)$ and $(V,g)$ respectively. This construction makes $\widetilde{\WLrd}$, $\widetilde{\DLrd}$, $\SLrdt$, and $\widetilde{\KLrd}$ with $\delta_{\square}$ metric spaces. We abuse terminology and refer to elements of the metric spaces $(\widetilde{\WLrd},\delta_{\square})$, $(\widetilde{\DLrd},\delta_{\square})$, $(\SLrdt,\delta_{\square})$, and $(\widetilde{\KLrd},\delta_{\square})$ as graphon signals, directed graphon signals, symmetric kernel signals, and general kernel signals, respectively.
\subsection{Regularity Lemmas}\label{sec:reg_lem}
Here, we extend \cite[Corollary B.11]{signal23} from undirected graphon-signals with one-dimensional signals to both directed and undirected graphon-signals, as well as symmetric and non-symmetric kernel-signals, with multidimensional signals. As in \cite{signal23}, we phrase the result in terms of the cut metric, which, in this context, is stronger than the cut distance. \cite[Corollary B.11]{signal23} asserts that any graphon-signal can be approximated by averaging both the graphon and the signal over an appropriate partition. While \cite{signal23} contains several versions of the regularity lemma, we extend only the one we find particularly important. The remaining versions can be extended in a similar manner with minimal modifications. To formulate our version of the regularity lemma, we begin by defining the relevant spaces of step functions.
\begin{definition}
\label{def:step}
Given a partition $\mathcal{P}_k$, and $d\in\NN$, we define the space $\mathcal{S}^{p\to d}_{\mathcal{P}_k}$ of \emph{step functions} $\R^p\mapsto\R^d$ over the partition $\mathcal{P}_k$ to be the space of functions $F:[0,1]^p\rightarrow\RR^d$ of the form
\begin{equation}
\label{eq:Sd}
   F(x_1,\ldots,x_p)=\sum_{j=(j_1,\ldots,j_p)\in [k]^p} \left(\prod_{l=1}^p\mathds{1}_{P_{j_l}}(x_l)\right)c_j, 
\end{equation}
   for any choice of $\{c_j\in \RR^d\}_{j\in [k]^p}$.
\end{definition}
Any element of $\cW_0 \cap \mathcal{S}^{2 \to 1}_{\mathcal{P}_k}$, $\cD_0 \cap \mathcal{S}^{2 \to 1}_{\mathcal{P}_k}$, $\cS_1 \cap \mathcal{S}^{2 \to 1}_{\mathcal{P}_k}$, and $\cW_1 \cap \mathcal{S}^{2 \to 1}_{\mathcal{P}_k}$ is called a \emph{step graphon}, a \emph{directed step graphon}, a \emph{symmetric step kernel}, and a \emph{step kernel}, respectively, with respect to $\mathcal{P}_k$. Similarly, any element of $\mathcal{L}_r^{\infty}([0,1]; \mathbb{R}^d) \cap \mathcal{S}^{1 \to d}_{\mathcal{P}_k}$ is called a \emph{step signal}. 

We define the spaces of step graphon-signals, directed step graphon-signals, symmetric step kernel-signals, and step kernel-signals with respect to $\mathcal{P}_k$, respectively, as
\begin{align*}
&[\WLrd]_{\mathcal{P}_k}:=(\cW_0\cap\mathcal{S}^{2\rightarrow 1}_{\mathcal{P}_k})\times (\mathcal{L}_r^{\infty}[0,1]\cap\mathcal{S}^{1\rightarrow d}_{\mathcal{P}_k}),\quad [\DLrd]_{\mathcal{P}_k}:=(\cD_0\cap\mathcal{S}^{2\rightarrow 1}_{\mathcal{P}_k})\times (\mathcal{L}_r^{\infty}[0,1]\cap\mathcal{S}^{1\rightarrow d}_{\mathcal{P}_k}),\\&
[\SLrd]_{\mathcal{P}_k}:=(\cS_1\cap\mathcal{S}^{2\rightarrow 1}_{\mathcal{P}_k})\times (\mathcal{L}_r^{\infty}[0,1]\cap\mathcal{S}^{1\rightarrow d}_{\mathcal{P}_k}),\\&
\hspace{6.6cm}{\rm and}\quad [\KLrd]_{\mathcal{P}_k}:=(\cW_1\cap\mathcal{S}^{2\rightarrow 1}_{\mathcal{P}_k})\times (\mathcal{L}_r^{\infty}[0,1]\cap\mathcal{S}^{1\rightarrow d}_{\mathcal{P}_k}).
\end{align*}
The projection, also called the \emph{stepping operator}, of a graphon-signal or kernel-signal onto a partition $\mathcal{P}_k$ is defined as follows.
\begin{restatable}[]{definition}{proj}
~\label{def:proj}
Let $\mathcal{P}_n=\{P_1,\ldots,P_n\}$ be a partition of $[0,1]$, and $(W,f) \in \mathcal{X}$, where  $\mathcal{X}$ is either $\WLrd$, $\DLrd$, $\SLrd$, or $\KLrd$. We define the \emph{projection} of $(W,f)$ upon $\X\cap\left(\mathcal{S}^{2\rightarrow 1}_{\mathcal{P}_n}\times\mathcal{S}^{1\rightarrow d}_{\mathcal{P}_n}\right)$ to be the step graphon-signal $(W,f)_{\mathcal{P}_n}=(W_{\mathcal{P}_n},f_{\mathcal{P}_n})$ that attains the value
\[
W_{\mathcal{P}_n}(x,y)
=\frac{1}{\mu(P_i)\mu(P_j)}
\int_{P_i\times P_j}W(u,v)\,du\,dv,
\qquad
f_{\mathcal{P}_n}(x)
=\frac{1}{\mu(P_i)}
\int_{P_i}f(u)\,du
\]
for every $(x,y)\in P_i\times P_j$ and $1\leq i,j\leq n$. Each formula applies when its denominator is positive; when that denominator is zero, define the corresponding projected value to be $0$.
\end{restatable}
Note that we use the notation $(W_{\mathcal{P}_n},f_{\mathcal{P}_n})$ and the notation $([W]_{\mathcal{P}_n},[f]_{\mathcal{P}_n})$ interchangeably to denote the projection of a graphon-signal onto the partition $\mathcal{P}_n$. 
\begin{restatable}[Regularity Lemma for Graphon-Signals/kernel-signals]{theorem}{theoremRegularity}~\label{lem:gs-reg-lem3}
For any $c>1$, $r>0$, sufficiently small $\epsilon>0$, $n \geq 2^{\lceil \frac{8c}{\epsilon^2}\rceil}$ and $(W,f) \in \X$, where  $\X$ is either $\WLrd$, $\DLrd$, $\SLrd$, or $\KLrd$, given the equipartition, $\cI_n$, of $[0,1]$ into $n$ intervals, there exists a measure preserving bijection $\phi$, such that
\[
d_{\square}\left(~\big(W,f\big)^{\phi}~,~\big((W,f)^{\phi}\big)_{\cI_n}~\right)\leq \epsilon.
\]
\end{restatable}
The proof of \cref{lem:gs-reg-lem3} is given in \cref{ap:regularitylemma}. We remark that the stepping (or projection) of a graphon, directed graphon, symmetric kernel, or general kernel yields an object with the same type. That is, stepping preserves the symmetry and value range of the kernel. Consequently, the weak regularity lemma approximates each class of kernels by a step kernel of the same type.

As a direct consequence, for every kernel–signal (of any of the four types), and for every $\epsilon>0$, there exists a piecewise constant kernel–signal of the same type with respect to an equipartition of $[0,1]$ into $n\geq2^{\lceil \tfrac{8c}{\epsilon^2}\rceil}$ intervals, such that the cut distance between the two is at most $\epsilon$.
\subsection{Compactness and Covering Numbers}
\cite[Theorem 3.6]{signal23} established that the space of graphon-signals with one-dimensional signals is compact in the cut distance and can be covered by a finite number of balls of any fixed radius. We extend both the compactness result and the explicit bound on the covering number to spaces of graphon-signals and kernel-signals with multidimensional signals. For each fixed $r>0$ and $d\in\NN$, the covering bound has the same asymptotic form as in the one-dimensional case, but its small-error threshold may depend on $r$ and $d$.
 
\begin{restatable}[]{theorem}{compactnesstheorem}~\label{th:compactness}
Let $r>0$.
The metric spaces $(\widetilde{\WLrd},\delta_{\square})$, $(\widetilde{\DLrd},\delta_{\square})$, $(\SLrdt,\delta_{\square})$, and $(\widetilde{\KLrd},\delta_{\square})$ and the pseudometric spaces $(\WLrd,\delta_{\square})$, $(\DLrd,\delta_{\square})$, $(\SLrd,\delta_{\square})$ and $(\KLrd,\delta_{\square})$ are compact.

Moreover,  
given any  $c>1$, for every sufficiently small $\epsilon>0$ (with upper bound depending on $c,r,d$), each one of these spaces can be covered by $\kappa(\epsilon)=  2^{k^2}$ closed balls of radius $\epsilon$, where $k=\lceil 2^{\frac{2c}{\epsilon^2}}\rceil$.
\end{restatable}

The proof of \cref{th:compactness} is given in \cref{ap:compactness}.

\subsection{Sampling Lemmas}\label{sec:sampling}
\cite[Theorem 3.7]{signal23} shows that a graph with $1$D signal, which is sub-sampled from a graphon with 1D signal by uniformly randomly sampling \(k\) points from \([0,1]\),  is close in high probability to the original graphon signal in the cut distance. Specifically, the expected cut distance between the graphon-signal and its sampled version decays at a rate of \(1/\sqrt{\log(k)}\). We show that the same rate holds for each fixed signal dimension $d$, with a sample-size threshold that may depend on $r$ and $d$. The results are divided into two separate theorems: the first holds for both graphon-signals and kernel-signals, whereas the second holds only for graphon-signals.

\begin{definition}\label{def:random-graph}
Let $\Lambda=(\lambda_1,\ldots\lambda_k)\in [0,1]^k$ be $k$ independent uniform random samples from $[0,1]$, and $(W,f)\in\KLrd$ be a graphon-signal. 
In all sampling statements, we use kernels with Lebesgue measurable diagonal restrictions $x\mapsto W(x,x)$, modifying only the diagonal if necessary.
A \emph{random weighted graph} $W(\Lambda)$, is a weighted graph with $k$ nodes and edge weight $w_{i,j} = W(\lambda_i,\lambda_j)$ between node $i$ and node $j$.  
A \emph{random sampled signal} $f(\Lambda)$, is a signal with value $f_i=f(\lambda_i)$ at each node $i$.  
\end{definition}
Note that when the kernel–signal is an undirected graphon–signal, a directed graphon–signal, a symmetric kernel-signal, or a general kernel–signal, the sampled graph $(W(\Lambda), f(\Lambda))$ corresponds to an undirected, a directed, an undirected $[-1,1]$-weighted, or a directed $[-1,1]$-weighted random graph-signal, respectively.

With this construction in place, we can state our first sampling lemma, which shows that the sampled weighted graph and signal remain close to their continuous counterparts in expectation.
\begin{restatable}[First Sampling Lemma]
    {theorem}{Samplinglemma}~\label{lem:second-sampling-garphon-signal00}
   Let $r>1$. There exists a constant $K_0>0$ that depends on $r$ and $d$, such that for every $k \geq  K_0$, every  $(W,f)\in \KLrd$,  and for  $\Lambda=(\lambda_1,\ldots\lambda_k)\in [0,1]^k$  independent uniform random samples from $[0,1]$, we have
\[
\mathbb{E}\bigg(\delta_{\square}\Big(\big(W,f\big),\big(W(\Lambda),f(\Lambda)\big)\Big)\bigg)  < \frac{15}{\sqrt{\log(k)}}.
\]
\end{restatable}
Theorem~\ref{lem:second-sampling-garphon-signal00} also applies to graphon–signals, directed graphon–signals, and undirected kernel–signals as special cases.

We now move from the weighted setting to the unweighted one. For this purpose, we introduce the notion of a random simple graph sampled from a graphon-signal.
\begin{definition}
Let $\Lambda=(\lambda_1,\ldots\lambda_k)\in [0,1]^k$ be $k$ independent uniform random samples from $[0,1]$, and $(W,f)$ be a directed or undirected graphon-signal. A \emph{random simple graph} $\mathbb{G}(W,\Lambda)$ is defined as a simple directed graph, with an edge from node $i$ to $j$ if and only if $e_{i,j}=1$, where, conditionally on $\Lambda$, the variables $\{e_{i,j}:i\ne j\}$ are independent Bernoulli variables with parameters $W(\lambda_i,\lambda_j)$, and $e_{i,i}=0$.
\end{definition}
If $W$ is a graphon, then $\mathbb{G}(W,\Lambda)$ is not necessarily an undirected graph-valued random variable. 
A simple way to obtain an undirected sample is to first generate a directed graph according to $W$. 
Let $A$ denote the adjacency matrix of the resulting directed graph. 
For symmetric $W$, we return the matrix $B$ defined by $B_{i,j}=B_{j,i}=A_{i,j}$ for $i<j$ and $B_{i,i}=0$.

Analogously  to the weighted case, the following lemma shows that the sampled simple graph and signal remain close in expectation to the original graphon-signal.  
\begin{restatable}[Second Sampling Lemma]
    {theorem}{SamplinglemmaTwo}~\label{lem:second-sampling-garphon-signal02}
   Let $r>1$. There exists a constant $K_0>0$ that depends on $r$ and $d$, such that for every $k \geq  K_0$, every $(W,f)\in\DLrd$ (including symmetric graphons), and for  $\Lambda=(\lambda_1,\ldots\lambda_k)\in [0,1]^k$  independent uniform random samples from $[0,1]$, we have
\[
\mathbb{E}\bigg(\delta_{\square}\Big(\big(W,f\big),\big(\mathbb{G}(W,\Lambda),f(\Lambda)\big)\Big)\bigg)  < \frac{15}{\sqrt{\log(k)}}.
\]
\end{restatable}

The proofs of \cref{lem:second-sampling-garphon-signal00,lem:second-sampling-garphon-signal02} are given in \cref{ap:sampling}.
\section{Extended Graphon-Signal Analysis of MPNNs}
In \cite[Section 4]{signal23}, Levie presents three results: (i) he proves the Lipschitz continuity of MPNNs \cite[Theorem 4.1]{signal23}, (ii) he establishes generalization bounds \cite[Theorem 4.2]{signal23} in a classification setting with labels given by a deterministic function of the input, and, (iii) he demonstrates the stability of MPNNs under sub-sampling of graph-signals \cite[Theorem 4.3]{signal23}. In this section, we extend these results from graphon-signals with one-dimensional node features to directed or undirected graphon-signals, as well as symmetric and general kernel-signals, all with multidimensional signals. An exception is the stability of MPNNs to graph-signal sub-sampling, which is not extended to kernels. In addition, \cite[Theorem 4.1]{signal23} establishes Lipschitz continuity only for MPNNs without a readout layer. Here, we show that MPNNs with and without readout layers remain Lipschitz continuous with respect to the cut distance. Furthermore, at a fixed covering radius and confidence level, we improve the dependence on the covering number in the uniform Monte Carlo estimate of \cite[Appendix G.3]{signal23} using the Bretagnolle–Huber–Carol inequality %
\cite[Proposition A.6.6]{Vaart2000}, following an approach similar to that of \cite[Theorem 3]{Xu2012}. Moreover, our bounds allow arbitrary joint distributions of inputs and labels. All proofs are deferred to \cref{Ap:lip,AP:gen,AP:stab}.
\subsection{Lipschitz Continuity of MPNNs}\label{sec:lip}
Here, we extend \cite[Theorem 4.1]{signal23}, which states that MPNNs without readout are Lipschitz continuous over the space of graphon-signals with 1D signal, to graphon-signals and kernel-signals with multidimensional signals. %

We define the formal bias of a function $f: \R^{d_1}\mapsto\R^{d_2}$ to be $\norm{f(0)}_\infty$ \cite{maskey2022generalization}.

\begin{restatable}[]{theorem}{Lipschitness}~\label{theorem:Lip}
Let $T,K$ be positive integers and $r,B,L>0$. Fix positive integer node dimensions $d_0=d,\ldots,d_T$, message dimensions $p_0,\ldots,p_{T-1}$, and readout dimension $d'$, where $p_j$ is the message dimension for the transition from layer $j$ to layer $j+1$.
Let $\mathcal{H}$ be the class of all $T$-layer MPNNs with readout $\Theta:\X\to\R^{d'}$ and these dimensions, where $\mathcal{X}$ is either $\WLrdt$, $\DLrdt$, $\SLrdt$, or $\KLrdt$. For $t\in[T]$, let $\mathcal{H}_t$ consist of their depth-$t$ node-feature maps.
For bounded signals $h,k$ of the same dimension, write
\[
\delta_\square(h,k):=\inf_{\phi\in S'_{[0,1]}}\norm{h-k^\phi}_\square.
\]
Node-feature maps are evaluated on representatives of $\X$, and their outputs are identified when this signal distance is zero. Impose the following constraints:
\begin{itemize}
    \item For every $t\in[T]$, $k\in[K]$ and $\gamma\in\{{\rm trans},{\rm rec}\}$, the maps $\psi$, $\mu^{(t)}$ and $\xi^{(t)}_{k,\gamma}$ are $L$-Lipschitz both for the $\ell^1$ norms and for the $\ell^\infty$ norms on their domains and codomains.
    \item For every $t\in[T]$, $k\in[K]$ and $\gamma\in\{{\rm trans},{\rm rec}\}$, the formal biases $\norm{\mu^{(t)}(0,0)}_\infty$, $\norm{\xi^{(t)}_{k,\gamma}(0)}_\infty$ and $\norm{\psi(0)}_\infty$ are bounded by $B$.
\end{itemize}
Then there exist constants $L_{\mathcal{H}_t}$, $t\in[T]$, $L_{\mathcal{H}}$, and $B_{\mathcal{H}}$, depending only on the fixed parameters and dimensions, such that for every $t\in[T]$, $\Theta_t\in\mathcal{H}_t$, $\Theta\in\mathcal{H}$ and $(W,f),(V,g)\in\X$, we have
\begin{enumerate}
    \item $\delta_\square(\Theta_t(W,f),\Theta_t(V,g)) \leq L_{\mathcal{H}_t} \cdot \delta_{\square}((W,f),(V,g))$.
    \item $\|\Theta(W,f) - \Theta(V,g)\|_\infty \leq L_{\mathcal{H}} \cdot \delta_{\square}((W,f),(V,g))$.
    \item $\|\Theta(W,f)\|_\infty \leq B_{\mathcal{H}}.$
\end{enumerate}
For fixed $T$ and dimensions, these constants admit polynomial bounds in $K,B,L,\max\{1,r\}$, with degrees depending on $T$.
\end{restatable}
The proof of \cref{theorem:Lip} is given in \cref{Ap:lip}. 
\subsection{A Generalization Theorem for MPNNs}
\Cref{th:compactness,theorem:Lip} establish two key properties of the graphon-signal domain: the space of graphon-signals admits a finite covering number, and MPNNs are Lipschitz continuous over this domain. Together, these properties yield a uniform generalization bound, akin to robustness-based guarantees \cite{Xu2012}. \cite[Theorem 4.2]{signal23} proves generalization bounds in a classification setting with labels given by a deterministic function of the input. Here, we extend that result to the case of graphon-signals and kernel-signals with multidimensional signals. Furthermore, at a fixed covering radius and confidence level, we improve the dependence on the covering number in the uniform Monte Carlo estimate of \cite[Appendix G.3]{signal23} using the Bretagnolle–Huber–Carol inequality \cite[Proposition A.6.6]{Vaart2000}, following the methodology of \cite[Theorem 3]{Xu2012}. Finally, our generalization bound allows arbitrary joint distributions of inputs and labels, without requiring labels to be deterministic functions of the inputs.

Consider $C$-class classification, where the data is drawn from a distribution over $(\X \times \{0,1\}^C, \Sigma, \nu)$, with $\Sigma$ denoting the Borel $\sigma$-algebra, $\nu$ a  Borel probability measure and $\X$ is either $\WLrdt$, $\DLrdt$, $\SLrdt$, or $\KLrdt$.\footnote{Here, $\Sigma$ is the Borel $\sigma$-algebra of the product space of $\X$ (with the cut distance topology) and $\{0,1\}^C$ with the discrete topology (all sets are open).}  Prediction and label vectors are measured in the $\ell^\infty$ norm. Equip $\mathcal{P}:=\X\times\{0,1\}^C$ with the metric
\[
\rho((x,c),(x',c')):=\delta_\square(x,x')+\|c-c'\|_{\ell^\infty}.
\]
Let $\mathcal{E}:\R^{d'}\times\{0,1\}^C\to\R$ be a loss function satisfying
\[
|\mathcal{E}(u,c)-\mathcal{E}(v,c')|
\le L_{\mathcal{E}}\bigl(\|u-v\|_{\ell^\infty}
+\|c-c'\|_{\ell^\infty}\bigr)
\]
whenever $\|u\|_{\ell^\infty},\|v\|_{\ell^\infty}\le B_{\mathcal{H}}$ and $c,c'\in\{0,1\}^C$, where $B_{\mathcal{H}}$ is the bound on the hypothesis class defined below.
\label{p:lip}Define ${\rm Lip}(\X, \R^{d'}, L_{\mathcal{H}}, B_{\mathcal{H}})$ to be the set of all continuous functions $f \colon \X \to \R^{d'}$ that are $L_{\mathcal{H}}$-Lipschitz and bounded by $B_{\mathcal{H}}$. We denote this class in short by ${\rm Lip}$. Since our hypothesis class $\mathcal{H}$ of $T$-layer MPNNs with readout is a subset of ${\rm Lip}$ (see \cref{sec:lip}), any generalization bound proven for the class ${\rm Lip}$ also applies to the class $\mathcal{H}$. For every function $\Upsilon$ in the hypothesis class ${\rm Lip}$, 
define the \emph{statistical risk} as \(\mathcal{R}(\Upsilon)
= \mathbb{E}_{(x,c)\sim\nu}\!\left[ \mathcal{E}(\Upsilon(x),c) \right] = \int \mathcal{E}(\Upsilon(x),c)\, d\nu(x,c).\) Given i.i.d. samples $\mathbf{X}:=\left((X_i,C_i)\right)_{i=1}^N \sim \nu^N$, we define the \emph{empirical risk} as \(\hat{\mathcal{R}}(\Upsilon_{\mathbf{X}},\mathbf{X}) 
= \frac{1}{N}\sum_{i=1}^{N}\, \mathcal{E}\!\left(\Upsilon_{\mathbf{X}}(X_i),\, C_i\right).\)
\begin{restatable}[MPNN generalization theorem]{theorem}{MPNNsGeneralizationTheorem}~\label{theorem:generalization}
Consider the above classification setting and put
\[
L:=L_{\mathcal{E}}\max\{L_{\mathcal{H}},1\},
\qquad
{\rm B}:=L_{\mathcal{E}}(B_{\mathcal{H}}+1)+|\mathcal{E}(0,0)|.
\]
Let $C,N\ge1$ be integers and let $\mathbf{X}:=((X_i,C_i))_{i=1}^N$ be drawn i.i.d. from $\nu$ on $\mathcal{P}=\X\times\{0,1\}^C$. Write $\kappa_{\X}(\epsilon)$ for the covering number of $(\X,\delta_\square)$ by closed balls of radius $\epsilon$ centered in $\X$; it is finite by \cref{th:compactness}.
For every fixed $\epsilon>0$ and $0<p\le1$, there exists an event $\mathcal{E}^{p,\epsilon}\in\Sigma^{\otimes N}$ with $\nu^N(\mathcal{E}^{p,\epsilon})\ge1-p$ on which, simultaneously for every $\Upsilon\in\mathrm{Lip}$,
\[
\begin{aligned}
&\left|\mathcal{R}(\Upsilon)-\hat{\mathcal{R}}(\Upsilon,\mathbf{X})\right|\\
&\quad\le
2\epsilon L+{\rm B}
\sqrt{\frac{2\bigl(2^C\kappa_{\X}(\epsilon)\ln2+\ln(1/p)\bigr)}{N}}.
\end{aligned}
\]
The radius $\epsilon$ is chosen independently of the sample. In particular, the bound holds for any sample-dependent choice $\Upsilon_{\mathbf{X}}$ in this class, including MPNNs in $\mathcal{H}$.
\end{restatable}
The proof of \cref{theorem:generalization} is given in \cref{AP:gen}.

\subsection{Stability of MPNNs to graph-signal sub-sampling}
Here, we present an extension of the stability of MPNNs to subsampling, as established in \cite[Theorem 4.3]{signal23}, to both directed and undirected graphon-signals with multidimensional signals. When dealing with large graphs, a common practice is to extract a smaller subgraph and apply a Message Passing Neural Network (MPNN) to the subsampled version \cite{GraphSAGE,chen2018fastgcn,ClusterGCN}. \cite[Theorem 4.3]{signal23} demonstrates that this approach holds theoretical validity: any Lipschitz continuous MPNN yields approximately the same results on both the original large graph and its subsampled counterpart. %
\begin{restatable}[]{theorem}{theoremstability}
~\label{thm:MPNN_samp}
Let $r>1$, let $d\ge1$ be an integer, and let $\Theta$ be an MPNN without readout defined on $\DLrd$. Let $L>0$ be a Lipschitz constant of the joint map
\[
(W,f)\longmapsto\big(W,\Theta(W,f)\big)
\]
with respect to the kernel-signal cut distance $\delta_\square$.
Define
\[
\Sigma = \big(W,\Theta(W,f)\big), \qquad
\Sigma(\Lambda) = \Big(\mathbb{G}(W,\Lambda), \Theta\big(\mathbb{G}(W,\Lambda), f(\Lambda)\big)\Big),
\]
for $(W,f)\in\DLrd$ and $\Lambda=(\lambda_1,\ldots,\lambda_k)$ drawn independently and uniformly from $[0,1]$.
Then there exists a constant $K_0>0$, depending only on $r,d$, such that for all $k\ge K_0$ and all $(W,f)\in\DLrd$,
\[
\mathbb{E}\Bigl[\,\delta_{\square}\bigl(\Sigma,\Sigma(\Lambda)\bigr)\Bigr]
<\frac{15L}{\sqrt{\log k}}.
\]
The expectation is over both $\Lambda$ and the conditional Bernoulli edge draws defining $\mathbb{G}(W,\Lambda)$.
\end{restatable}
The proof of \cref{thm:MPNN_samp} is given in \cref{AP:stab}.

\section{Conclusion}
We refined and extended the graphon-signal framework in \cite{signal23} for analyzing message passing graph neural networks (MPNNs). Specifically, we generalized the main results to directed graphon-signals and krernel0signals  with multidimensional node attributes, extended the Lipschitz continuity to MPNNs with readout layers under the graphon-signal cut distance, and improved the generalization bounds using robustness arguments. %
\newpage
\bibliographystyle{abbrv}
\bibliography{bib_thesis}

@string{icml = "Int. Conference on Machine Learning (ICML)"}

@string{colt = "Conference on Learning Theory (COLT)"}

@string{nips = "Advances in Neural Information Processing Systems (NIPS)"}

@string{neurips = "Advances in Neural Information Processing Systems (NeurIPS)"}

@string{jmlr = "Journal of Machine Learning Research"}

@string{iclr = "Int. Conf. on Learning Representations (ICLR)"}

@article{signal23,
    author    = {Ron Levie},
    title     = {A Graphon-Signal Analysis of Graph Neural Networks},
    year      = {2023},
    journal = neurips
}

@article{bio02,
  author    = {Kenneth Atz and Francesca Grisoni and Gisbert Schneider},
  title     = {Geometric deep learning on molecular representations},
  journal   = {Nature Machine Intelligence},
  year      = {2021},
  volume    = {3},
  number    = {12},
  pages     = {1023--1032},
  doi       = {10.1038/s42256-021-00418-8}
}

@book{statBackground,
  author    = {Shai Shalev-Shwartz and Shai Ben-David},
  title     = {Understanding Machine Learning - From Theory to Algorithms},
  year      = {2014},
  publisher = {Cambridge University Press},
  pages     = {I-XVI, 1--397},
  isbn      = {978-1-10-705713-5}
}

@article{comptreeImportant1,
author = {Christopher Morris and Martin Ritzert and Matthias Fey and William L. Hamilton and Jan Eric Lenssen and Gaurav Rattan and Martin Grohe},
    title     = {{W}eisfeiler and {L}eman Go Neural: Higher-order Graph Neural Networks},
    journal = {The Thirty-Third AAAI Conference on Artificial Intelligence (AAAI-19)},
    year      = {2019}
}

@inproceedings{maskey2022generalization,
  author    = {Sohir Maskey and Ron Levie and Yunseok Lee and Gitta Kutyniok},
  title     = {Generalization Analysis of Message Passing Neural Networks on Large Random Graphs},
  booktitle = neurips,
  year      = {2022},
  publisher = {Curran Associates, Inc.}
}

@book{lovasz2012large,
  author    = {L. Lovász},
  title     = {Large Networks and Graph Limits},
  volume    = {60},
  series    = {Colloquium Publications},
  publisher = {American Mathematical Society},
  year      = {2012}
}

@article{frieze1999quick,
  author    = {A. Frieze and R. Kannan},
  title     = {Quick Approximation to Matrices and Applications},
  journal   = {Combinatorica},
  volume    = {19},
  number    = {2},
  pages     = {175--220},
  year      = {1999},
  month     = {Feb.}
}

@article{xu2018powerful,
  title={How Powerful are Graph Neural Networks?},
  author={Keyulu Xu and  Weihua Hu and Jure Leskovec and Stefanie Jegelka},
  journal=iclr,
  year={2019}
}

@article{keisler2022forecasting,
  title={Forecasting Global Weather with Graph Neural Networks},
  author={Ryan Keisler},
  journal={arXiv preprint arXiv:2202.07575},
  year={2022}
}

@Inbook{wang2022graph,
author="Yuyang Wang and Zijie Li and Amir Barati Farimani",
title="Graph Neural Networks for Molecules",
bookTitle="Machine Learning in Molecular Sciences",
year="2023",
publisher="Springer International Publishing",
address="Cham",
pages="21--66",
isbn="978-3-031-37196-7",
doi="10.1007/978-3-031-37196-7_2"
}

@inproceedings{fan2019graph,
author = {Fan, Wenqi and Ma, Yao and Li, Qing and He, Yuan and Zhao, Eric and Tang, Jiliang and Yin, Dawei},
title = {Graph Neural Networks for Social Recommendation},
year = {2019},
isbn = {9781450366748},
publisher = {Association for Computing Machinery},
address = {New York, NY, USA},
url = {https://doi.org/10.1145/3308558.3313488},
doi = {10.1145/3308558.3313488},
booktitle = {The World Wide Web Conference},
pages = {417–426},
numpages = {10},
location = {San Francisco, CA, USA},
series = {WWW '19}
}

@article{zhou2018graph,
title = {Graph neural networks: A review of methods and applications},
journal = {AI Open},
volume = {1},
pages = {57-81},
year = {2020},
issn = {2666-6510},
doi = {https://doi.org/10.1016/j.aiopen.2021.01.001},
author = {Jie Zhou and Ganqu Cui and Shengding Hu and Zhengyan Zhang and Cheng Yang and Zhiyuan Liu and Lifeng Wang and Changcheng Li and Maosong Sun}
}

@InProceedings{qian2023exploring,
  title = 	 {Exploring the Power of Graph Neural Networks in Solving Linear Optimization Problems},
  author =       {Qian, Chendi and Ch\'{e}telat, Didier and Morris, Christopher},
  booktitle = 	 {Proceedings of The 27th International Conference on Artificial Intelligence and Statistics},
  pages = 	 {1432--1440},
  year = 	 {2024},
  editor = 	 {Dasgupta, Sanjoy and Mandt, Stephan and Li, Yingzhen},
  volume = 	 {238},
  series = 	 {Proceedings of Machine Learning Research},
  month = 	 {02--04 May},
  publisher =    {PMLR}
}

@article{yang2023ptgb,
  title={{PTGB}: Pre-Train Graph Neural Networks for Brain Network Analysis},
  author={Yi Yang and Hejie Cui and Carl Yang},
  journal={Conference on Health, Inference, and Learning (CHIL)},
  year={2023}
}

@article{cappart2023combinatorial,
  title={Combinatorial Optimization and Reasoning with Graph Neural Networks},
  author={Cappart, Quentin and Ch{\'e}telat, Didier and Khalil, Elias and Lodi, Andrea and Morris, Christopher and Veli{\v{c}}kovi{\'c}, Petar},
  journal= jmlr,
  volume={24},
  pages={1--61},
  year={2023},
  archivePrefix={arXiv},
  eprint={2102.09544},
  primaryClass={cs.LG}
}

@article{STOKES2020688,
title = {A Deep Learning Approach to Antibiotic Discovery},
journal = {Cell},
volume = {180},
number = {4},
pages = {688-702.e13},
year = {2020},
issn = {0092-8674},
doi = {https://doi.org/10.1016/j.cell.2020.01.021},
author = {Jonathan M. Stokes and Kevin Yang and Kyle Swanson and Wengong Jin and Andres Cubillos-Ruiz and Nina M. Donghia and Craig R. MacNair and Shawn French and Lindsey A. Carfrae and Zohar Bloom-Ackermann and Victoria M. Tran and Anush Chiappino-Pepe and Ahmed H. Badran and Ian W. Andrews and Emma J. Chory and George M. Church and Eric D. Brown and Tommi S. Jaakkola and Regina Barzilay and James J. Collins}
}

@article{verma2019stability,
  title={Stability and Generalization of Graph Convolutional Neural Networks},
  author={Verma, Saurabh and Zhang, Zhi-Li},
  journal={The 25th ACM SIGKDD Conference on Knowledge Discovery and Data Mining},
  year={2019},
  archivePrefix={arXiv},
  eprint={1905.01004},
  primaryClass={cs.LG}
}

@article{yehudai2020local,
  title={From Local Structures to Size Generalization in Graph Neural Networks},
  author={Gilad Yehudai and Ethan Fetaya and Eli Meir and  Gal Chechik and Haggai Maron},
  journal=icml,
  year={2020}
}

@book{Vaart2000,
  author    = {A. W. van der Vaart and J. A. Wellner},
  title     = {Weak Convergence and Empirical Processes},
  publisher = {Springer-Verlag},
  address   = {New York},
  year      = {2000}
}

@article{ruiz2021graph,
  author={Luana Ruiz and Fernando Gama and Alejandro Ribeiro},
  journal={Proceedings of the IEEE}, 
  title={Graph Neural Networks: Architectures, Stability, and Transferability}, 
  year={2021},
  volume={109},
  number={5},
  pages={660-682},
  doi={10.1109/JPROC.2021.3055400}}

@article{Xu2012,
  author    = {Huan Xu and Shie Mannor},
  title     = {Robustness and generalization},
  journal   = colt,
  year      = {2012},
  volume    = {86},
  number    = {3},
  pages     = {391--423},
  doi       = {10.1007/s10994-011-5268-1},
  issn      = {1573-0565}
}

@misc{maskey2024generalizationboundsmessagepassing,
      title={Generalization Bounds for Message Passing Networks on Mixture of Graphons}, 
      author={Sohir Maskey and Gitta Kutyniok and Ron Levie},
      year={2024},
      eprint={2404.03473},
      archivePrefix={arXiv},
      primaryClass={cs.LG}, 
}

@inproceedings{NEURIPS2020_dab49080,
 author = {Oono, Kenta and Suzuki, Taiji},
 booktitle = {Advances in Neural Information Processing Systems},
 editor = {H. Larochelle and M. Ranzato and R. Hadsell and M.F. Balcan and H. Lin},
 pages = {18917--18930},
 publisher = {Curran Associates, Inc.},
 title = {Optimization and Generalization Analysis of Transduction through Gradient Boosting and Application to Multi-scale Graph Neural Networks},
 volume = {33},
 year = {2020}
}

@inproceedings{Li2022GeneralizationGO,
  title={Generalization Guarantee of Training Graph Convolutional Networks with Graph Topology Sampling},
  author={Hongkang Li and M. Wang and Sijia Liu and Pin-Yu Chen and Jinjun Xiong},
  booktitle={International Conference on Machine Learning},
  year={2022}
}

@InProceedings{pmlr-v202-tang23f,
  title = 	 {Towards Understanding the Generalization of Graph Neural Networks},
  author =       {Huayi Tang and Yong Liu},
  booktitle = 	 icml,
  pages = 	 {33674--33719},
  year = 	 {2023},
  editor = 	 {Krause, Andreas and Brunskill, Emma and Cho, Kyunghyun and Engelhardt, Barbara and Sabato, Sivan and Scarlett, Jonathan},
  volume = 	 {202},
  series = 	 {Proceedings of Machine Learning Research},
  month = 	 {23--29 Jul},
  publisher =    {PMLR}
}

@article{DIDMs25,
    author = {Levi Rauchwerger and Stefanie Jegelka and Ron Levie},
    title = {Generalization, Expressivity, and Universality of Graph Neural Networks on Attributed Graphs},
    journal = iclr,
    year = {2025},
}

@article{chen2018fastgcn,
  author    = {J. Chen and T. Ma and C. Xiao},
  title     = {FastGCN: Fast learning with graph convolutional networks via importance sampling},
  journal = iclr,
  year      = {2018}
}

@inproceedings{GraphSAGE,
author = {Hamilton, William L. and Ying, Rex and Leskovec, Jure},
title = {Inductive Representation Learning on Large Graphs},
year = {2017},
isbn = {9781510860964},
publisher = {Curran Associates Inc.},
pages = {1025–1035},
numpages = {11},
location = {Long Beach, California, USA},
booktitle={Advances in Neural Information Processing Systems}
}

@inproceedings{ClusterGCN,
author = {Chiang, Wei-Lin and Liu, Xuanqing and Si, Si and Li, Yang and Bengio, Samy and Hsieh, Cho-Jui},
title = {Cluster-GCN: An Efficient Algorithm for Training Deep and Large Graph Convolutional Networks},
year = {2019},
isbn = {9781450362016},
publisher = {Association for Computing Machinery},
address = {New York, NY, USA},
doi = {10.1145/3292500.3330925},
booktitle = {Proceedings of the 25th ACM SIGKDD International Conference on Knowledge Discovery and Data Mining},
pages = {257–266},
numpages = {10},
location = {Anchorage, AK, USA},
series = {KDD '19}
}

@article{lovas2007,
author = {László, Lovász and Szegedy, Balázs},
year = {2007},
month = {04},
pages = {252-270},
title = {Szemerédi’s Lemma for the Analyst},
volume = {17},
journal = {Geometric and Functional Analysis},
doi = {10.1007/s00039-007-0599-6}
}

@inproceedings{10.5555/3737916.3739744,
author = {Finkelshtein, Ben and Ceylan, undefinedsmail undefinedlkan and Bronstein, Michael and Levie, Ron},
title = {Learning on large graphs using intersecting communities},
year = {2025},
isbn = {9798331314385},
publisher = {Curran Associates Inc.},
address = {Red Hook, NY, USA},
booktitle = {Proceedings of the 38th International Conference on Neural Information Processing Systems},
articleno = {1828},
numpages = {40},
location = {Vancouver, BC, Canada},
series = {NIPS '24}
}

@article{Rossi23,
author = {Rossi, Emanuele and Charpentier, Bertrand and Di Giovanni, Francesco and Frasca, Fabrizio and Günnemann, Stephan and Bronstein, Michael},
year = {2023},
month = {05},
pages = {},
title = {Edge Directionality Improves Learning on Heterophilic Graphs},
doi = {10.48550/arXiv.2305.10498},
journal={Proceedings of Machine Learning Research}
}

@book{folland1999real,
  title     = {Real Analysis: Modern Techniques and Their Applications},
  author    = {Folland, Gerald B.},
  year      = {1999},
  edition   = {Second},
  publisher = {Wiley},
  address   = {New York},
  note      = {A Wiley-Interscience publication; includes bibliographical references (pages 365--375) and index}
}

@inproceedings{defferrard2017convolutional,
author = {Defferrard, Micha\"{e}l and Bresson, Xavier and Vandergheynst, Pierre},
title = {Convolutional Neural Networks on Graphs with Fast Localized Spectral Filtering},
year = {2016},
isbn = {9781510838819},
publisher = {Curran Associates Inc.},
booktitle = {NeurIPS},
numpages = {9},
location = {Barcelona, Spain},
}

@inproceedings{kipf2017semirefsupervised,
  author = {Kipf, Thomas N. and Welling, Max},
  booktitle = {ICLR},
  location = {Palais des Congr{\`e}s Neptune, Toulon, France},
  title = {Semi-Supervised Classification with Graph Convolutional Networks},
  year = 2017
}

@ARTICLE{levie2018cayleynets,

  author={Levie, Ron and Monti, Federico and Bresson, Xavier and Bronstein, Michael M.},

  journal={IEEE Transactions on Signal Processing}, 

  title={CayleyNets: Graph Convolutional Neural Networks With Complex Rational Spectral Filters}, 

  year={2019},

  volume={67},

  number={1},

  pages={97-109},

  doi={10.1109/TSP.2018.2879624}}
\appendix
\newpage
\section{Properties of the Cut Norm}
\label{AP:Notation}
Here, we introduce some important properties of the cut norm from \cite{signal23, lovasz2012large}, one of which we extend to our setting (see \cref{claim:cutProp}).
\begin{lemma}[Cut-norm attainment]
\label{lem:sup_attained}
For every bounded Lebesgue measurable function $W:[0,1]^2\to\RR$, the supremum
\[
\sup_{S,T\subseteq[0,1]}
\left|\int_{S\times T}W(x,y)\,dx\,dy\right|,
\]
where $S,T$ range over measurable sets, is attained.
\end{lemma}
\begin{proof}
We adapt the maximization argument in \cite[Lemma 8.10]{lovasz2012large}, applying it to $W$ and $-W$. That argument uses boundedness and measurability but does not use symmetry, and yields measurable maximizers for each of the two signed integrals. Taking the larger of these two maxima proves the claim.
\end{proof}
For a measurable function \(W : [0,1]^2 \to [-r, r]\), the \(L^1\) and \(L^2\) norms are defined as  
\[
\|W\|_{1} := \int_{[0,1]^2} |W(x,y)| \, dx\,dy, \quad 
\|W\|_{2} := \left( \int_{[0,1]^2} |W(x,y)|^2 \, dx\,dy \right)^{1/2}.
\]
In \cite[Appendix A]{signal23}, Levie showed that for every \(r > 0\) and measurable function \(W : [0,1]^2 \to [-r, r]\), the following chain of inequalities holds:
\begin{align}\label{eq:cutnormequi}
0 \leq \|W\|_{\square} \leq \|W\|_1 \leq \|W\|_2 \leq \|W\|_{\infty} \leq r.
\end{align}
In addition, Levie shows that the signal cut norm is equivalent to the $L^1$ norm, i.e.,
\[
\frac{1}{2}\norm{f}_1\leq \norm{f}_{\square} \leq \norm{f}_1.
\]
We now extend this equivalency and show that both the multidimensional signal cut norm and the product signal cut norm are equivalent to the $L^1$ norm (see \cref{claim:cutProp}).
\normequivalency*
\begin{proof}
Write $f=(f_1,\ldots,f_d)$ and, for each $i\in[d]$, set
\[
a_i:=\sup_{S\subseteq[0,1]}\left|\int_S f_i(x)\,dx\right|,
\]
where the supremum is over measurable sets. Since $f_i\in\mathcal{L}^1([0,1];\R)$, its positive and negative parts $f_i^+:=\max\{f_i,0\}$ and $f_i^-:=\max\{-f_i,0\}$ are integrable. Taking $S=\{f_i>0\}$ or $S=\{f_i<0\}$ gives
\[
a_i=\max\left\{\int_0^1 f_i^+(x)\,dx,\int_0^1 f_i^-(x)\,dx\right\}.
\]
Consequently,
\[
\frac12\norm{f_i}_1\leq a_i\leq\norm{f_i}_1.
\]
By \cref{definition:productNorm},
\[
\norm{f}_{\square_\times}=\frac1d\sum_{i=1}^d a_i,
\]
and therefore
\[
\frac{1}{2d}\norm{f}_1
\leq\norm{f}_{\square_\times}
\leq\frac1d\norm{f}_1.
\]
By \cref{definition:multidimCutNorm},
\[
\norm{f}_{\square}
=\frac1d\sup_{S\subseteq[0,1]}\sum_{i=1}^d\left|\int_S f_i(x)\,dx\right|.
\]
Hence
\[
\frac1d\max_{i\in[d]}a_i
\leq\norm{f}_{\square}
\leq\frac1d\sum_{i=1}^d a_i
=\norm{f}_{\square_\times}.
\]
Since $\max_i a_i\geq d^{-1}\sum_i a_i$, it follows that
\[
\frac1d\norm{f}_{\square_\times}
\leq\norm{f}_{\square}
\leq\norm{f}_{\square_\times}.
\]
Combining these inequalities proves the claim.
\end{proof}
Lastly, we define $L^1$ norm of a graphon-signal which we later use in \cref{ap:compactness}.
\begin{definition}\label{def:l1graphonsignal}
For a graphon-signal $(W, f)$, we define its $L^1$ norm as
\[
\|(W, f)\|_{1} \;:=\; \|W\|_{1} \;+\; \|f\|_{1},
\]
\end{definition}

\section{Regularity Lemma for Graphon-Signals and Kernel-Signals}\label{ap:regularitylemma}
The following lemmas, originally presented in \cite{signal23,lovas2007}, are instrumental in proving the weak regularity lemma for graphon-signals and kernel-signals. Our results extend these lemmas to the setting of vector-valued functions. In particular, \cref{l:equipartition} generalizes \cite[Corollary B.11]{signal23} to functions with output in $\RR^d$. \cite[Corollary B.11]{signal23} asserts that any graphon-signal/kernel-signal can be approximated by averaging both the graphon and the signal over an appropriate partition.

All of the proofs largely follow the approach developed by Levie \cite{signal23}, but include necessary modifications to accommodate the multidimensional nature of the functions under consideration. Specifically, while Levie’s results pertain to scalar-valued functions, our extensions handle vector-valued functions by carefully adapting integration arguments, norm equivalences, and partitioning steps to ensure that they hold componentwise or in aggregate across dimensions. These adjustments, though conceptually straightforward, are crucial for preserving the structure and correctness of the original arguments in the more general setting.
\begin{lemma}[Equitizing partitions]
\label{l:equipartition}
Let $\mathcal{P}_k=\{P_1,\ldots,P_k\}$ be a measurable partition of $[0,1]$. For every integer $n>k$, there exists an equipartition $\mathcal{E}_n$ of $[0,1]$ into $n$ sets, depending only on $\mathcal{P}_k$ and $n$, such that, for every $p,d\in\mathbb{N}$ and every $F\in\mathcal{S}_{\mathcal{P}_k}^{p\to d}$, there exists $F'\in\mathcal{S}_{\mathcal{E}_n}^{p\to d}$ satisfying
\[
\norm{F-F'}_1\leq pd\,\norm{F}_\infty\frac{k}{n}.
\]
Here,
\[
\norm{F}_1:=\int_{[0,1]^p}\sum_{i=1}^d|F_i(x)|\,dx,
\qquad
\norm{F}_\infty:=\operatorname*{ess\,sup}_{x\in[0,1]^p}\norm{F(x)}_{\ell^\infty}.
\]
\end{lemma}
\begin{proof}
For each $i\in[k]$, set $q_i:=\max\{\lceil n\mu(P_i)\rceil-1,0\}$. Since Lebesgue measure is atomless, we can partition $P_i$ into $q_i$ sets $Q_{i,1},\ldots,Q_{i,q_i}$ of measure $1/n$ and a remainder $R_i$ of measure at most $1/n$. If $q_i=0$, this list of sets is empty and $R_i=P_i$. Let
\[
\mathcal{Q}:=\{Q_{i,j}:i\in[k],\ 1\leq j\leq q_i\},
\qquad
R:=\bigcup_{i=1}^k R_i,
\qquad
h:=n-\sum_{i=1}^k q_i.
\]
Then $\mu(R)=h/n\leq k/n$. Our choice of $q_i$ leaves a remainder of positive measure whenever $\mu(P_i)>0$, so $1\leq h\leq k$. Again using atomlessness, partition $R$ into $R'_1,\ldots,R'_h$, each of measure $1/n$. Thus
\[
\mathcal{E}_n:=\mathcal{Q}\cup\{R'_1,\ldots,R'_h\}
\]
is an equipartition into $n$ sets. Its construction does not depend on $p$, $d$, or $F$.

Fix $F\in\mathcal{S}_{\mathcal{P}_k}^{p\to d}$. On a product cell $Z_1\times\cdots\times Z_p$ of $\mathcal{E}_n$ with every $Z_\ell\in\mathcal{Q}$, each $Z_\ell$ lies in an original cell $P_{i_\ell}$; set $F'$ equal to the value of $F$ on $P_{i_1}\times\cdots\times P_{i_p}$. On all other product cells, set $F'=0$. This defines $F'\in\mathcal{S}_{\mathcal{E}_n}^{p\to d}$ with $\norm{F'}_\infty\leq\norm{F}_\infty$.

Let
\[
A:=\bigcup_{\ell=1}^p[0,1]^{\ell-1}\times R\times[0,1]^{p-\ell}.
\]
Then $F-F'$ vanishes outside $A$ and equals $F$ on $A$. By the union bound for product Lebesgue measure,
\[
\mu^{\otimes p}(A)\leq p\mu(R)\leq\frac{pk}{n}.
\]
Consequently,
\[
\begin{aligned}
\norm{F-F'}_1
&=\int_A\sum_{i=1}^d|F_i(x)|\,dx\\
&\leq d\,\norm{F}_\infty\,\mu^{\otimes p}(A)\\
&\leq pd\,\norm{F}_\infty\frac{k}{n}.
\end{aligned}
\]
\end{proof}
\begin{lemma}[\cite{signal23}, Lemma B.3]
\label{lem:all_inter}
   Let $\mathcal{S}=\{S_j\subset [0,1]\}_{j=1}^{m}$ be a collection of measurable sets (that are not disjoint in general), and $d\in\NN$. Let $\mathcal{C}_{\mathcal{S}}^{d\to 1}$ be the space of functions $F:[0,1]^{d}\to \RR$ of the form
   \[F(x)=\sum_{j=(j_1,\ldots,j_d)\in [m]^d} c_j\prod_{l=1}^d\mathds{1}_{S_{j_l}}(x_l),\]
   for some choice of $\{c_j\in \RR\}_{j\in [m]^d}$.
   Then, there exists a partition $\cP_k=\{P_1,\ldots,P_k\}$ into $k=2^m$ sets, that depends only on $\mathcal{S}$, such that
   \[\mathcal{C}_{\mathcal{S}}^{d\to 1} \subset \mathcal{S}^{d\to 1}_{\cP_k}.\]
   \end{lemma}
\begin{lemma}[\cite{signal23}, Lemma B.4]
\label{lem:all_inter1}
  Let $\cP_k=\{P_1,\ldots,P_k\},\mathcal{Q}_m=\{Q_1,\ldots,Q_m\} $ be two partitions. Then, there exists a partition $\mathcal{Z}_{km}$ into $km$ sets such that for every $d$,
   \[\mathcal{S}^{d\to1}_{\cP_k} \subset \mathcal{S}^{d\to1}_{\mathcal{Z}_{mk}} , \quad \text{and} \quad \mathcal{S}^{d\to1}_{\mathcal{Q}_m} \subset \mathcal{S}^{d\to1}_{\mathcal{Z}_{mk}}.\]
   \end{lemma}
Note that \cite[Lemma 4.1]{lovas2007} appears in \cite{signal23} as Lemma B.5.
\begin{lemma}[\cite{lovas2007}, Lemma 4.1]
\label{fact:szlemma}
Let $\mathcal{K}_1, \mathcal{K}_2,\ldots$ be arbitrary nonempty subsets (not necessarily subspaces) of a Hilbert space $\Hilbert$. Then, for every $\epsilon>0$ and $v\in \Hilbert$ there is  $m\leq \lceil 1/\epsilon^2 \rceil$  and  $v_i\in \mathcal{K}_i$ and $\gamma_i \in \RR$,  $i\in[m]$, such that for every $w\in \mathcal{K}_{m+1}$
\begin{align*}
    \bigg|\ip{w}{v-(\sum_{i=1}^{m}\gamma_i v_i)}\bigg|\leq \epsilon\ \norm{w}\norm{v}.
\end{align*}
\end{lemma}
To write a multidimensional version of \cite[Corollary B.11]{signal23}, we extend the following results, starting with formulating the regularity lemmas for graphons and kernels and then extending it to graphon-signals and kernel-signals. 
We first prove \cref{lem:gs-reg-lem02} and then incorporate signals in \cref{lem:gs-reg-lem0}.
\begin{theorem}[Weak regularity lemma for graphons/kernels]
\label{lem:gs-reg-lem02}
Let $\epsilon>0$. For every $W \in \X$, where $\X$ is either $\cW_0$, $\cD_0$, $\cS_1$, or $\cW_1$, there exists a partition $\cP_k$ of $[0,1]$ into $  k= 2^{2\lceil 1/\epsilon^2\rceil }$ sets, a step function $W_{k}\in \mathcal{S}_{\cP_k}^2\cap \X$, such that
\begin{equation*}
  \norm{W-W_{k}}_{\square}\leq \epsilon.
\end{equation*} 
Moreover, if $W$ is a step kernel with respect to an interval equipartition $\cI_n=\{I_1,\ldots,I_n\}$ of $[0,1]$, the approximation may instead be chosen on a partition with at most $k$ parts, each a union of intervals from $\cI_n$.
\end{theorem} 
\begin{proof}
Let $\mathcal{P}_0=\{[0,1]\}$ and let $U_j$ be the cell averages of $W$ on $\mathcal{P}_j\times\mathcal{P}_j$, with value zero on null product cells. Write $\norm{\cdot}_2$ for the $L^2([0,1]^2)$ norm. If $\norm{W-U_j}_{\square}>\epsilon$, choose measurable $A,B\subseteq[0,1]$ with $\left|\int_{A\times B}(W-U_j)\right|>\epsilon$, and refine $\mathcal{P}_j$ by $A$ and $B$, discarding empty parts. The resulting partition $\mathcal{P}_{j+1}$ has at most four times as many parts. Since $A\times B$ is a union of refined product cells, orthogonality of cell averaging and Cauchy--Schwarz give
\[
\norm{U_{j+1}}_2^2-\norm{U_j}_2^2
=\norm{U_{j+1}-U_j}_2^2
\geq\left|\int_{A\times B}(W-U_j)\right|^2
>\epsilon^2.
\]
Since $\norm{U_j}_2^2\leq\norm{W}_2^2\leq1$, the process stops within $\lceil1/\epsilon^2\rceil$ refinements, with cut error at most $\epsilon$ and at most $4^{\lceil1/\epsilon^2\rceil}=k$ parts. Cell averages preserve the required range and symmetry. Keeping the resulting $U_j$ unchanged, split positive-measure parts if necessary to obtain exactly $k$ parts, and set $W_k=U_j$.

If $W$ is a step kernel on $\cI_n$ and the parts of $\mathcal{P}_j$ are unions of its intervals, the residual $W-U_j$ is constant on each $I_a\times I_b$. Its rectangular integral is bilinear in the fractions of each interval included in $A$ and $B$. Its absolute value is convex in each fraction separately, so its maximum is attained with all fractions in $\{0,1\}$. Thus $A,B$ may be chosen as unions of intervals from $\cI_n$, and every refinement retains this property. Omitting the final splitting gives the special assertion with at most $k$ parts.
\end{proof}
We now extend \cref{lem:gs-reg-lem02} to \cref{lem:gs-reg-lem0} by analysing the signals part.
\begin{theorem}[Weak regularity lemma for graphon-signals/kernel-signals]
\label{lem:gs-reg-lem0}
Let $\epsilon,\rho>0$. For every $(W,f) \in \X$, where $\X$ is either $\WLrd$, $\DLrd$, $\SLrd$, or $\KLrd$, there exists a partition $\cP_k$ of $[0,1]$ into $  k=\lceil r/\rho \rceil^d \Big( 2^{2\lceil 1/\epsilon^2\rceil }\Big)$ 
 sets, a step function $W_{k}\in \mathcal{S}_{\cP_k}^2\cap \Y$, where $\Y$ is either $\cW_0$, $\cD_0$, $\cS_1$, or $\cW_1$, respectively, and a step function signal $f_{k}\in \mathcal{S}_{\cP_k}^{1\to d}\cap \cL^{\infty}_r([0,1];\R^d)$, such that
\begin{equation*}
  \norm{W-W_{k}}_{\square}\leq \epsilon  \quad \text{and}\;\; \norm{f-f_{k}}_{\square} \leq \rho .
\end{equation*} 
\end{theorem} 
\begin{proof}
Changing $f$ on a null set if necessary, assume $f(x)\in[-r,r]^d$ for every $x$. To analyze the signal part, we analyze each signal channel $f_m$ separately. Let $m\in[d]$, we partition the range of the channel $[-r,r]$ into $j=\lceil r/\rho \rceil$ intervals $\{J_i\}_{i=1}^j$ of length less or equal to $2\rho$, where the left edge point of each $J_i$ is $-r+(i-1)\frac{2r}{j}$. Consider the partition of $[0,1]$ based on the preimages $\mathcal{Y}_j=\{Y_i=f_m^{-1}(J_i)\}_{i=1}^j$. Clearly, for the step signal
\[f_{\rho,m}(x) = \sum_{i=1}^j a_i\mathds{1}_{Y_i}(x),\]
where $a_i$ the midpoint of the interval $J_i$,
we have
\[\norm{f_m-f_{\rho,m}}_{\square} \leq\norm{f_m-f_{\rho,m}}_1 \leq \rho.\]
Define $f_\rho=(f_{\rho,1},\ldots,f_{\rho,d})$. This is true for any $m\in[d]$, thus,
\[\norm{f-f_{\rho}}_{\square} \leq\frac{1}{d}\norm{f-f_{\rho}}_1 \leq \rho.\]
Lastly, by \cref{lem:gs-reg-lem02,lem:all_inter1} (refining further if necessary), there is a partition $\mathcal{P}_{k}$  of $[0,1]$ into $  k=\lceil r/\rho \rceil^d \Big( 2^{2\lceil 1/\epsilon^2\rceil }\Big)$ 
 sets such that $W_{\epsilon}\in \mathcal{S}^2_{\mathcal{P}_k}$ and $f_{\rho}\in \mathcal{S}^{1\to d}_{\mathcal{P}_k}$. 
\end{proof}
\cref{lem:gs-reg-lem00} extends \cite[Corollary B.7]{signal23} and follows the same proof strategy.
\begin{corollary}[Weak regularity lemma for graphon-signals/kernel-signals -- version 2]
\label{lem:gs-reg-lem00}
Let $r>0$ and $c>1$. For every sufficiently small  $\epsilon>0$ (namely, $\epsilon$ that satisfies \cref{eq:epsilon_opt0}), and for every $(W,f) \in \X$, where $\X$ $\mathcal{X}$ is either $\WLrd$, $\DLrd$, $\SLrd$, or $\KLrd$, there exists a partition $\cP_k$ of $[0,1]$ into $k= \Big( 2^{\lceil 2c/\epsilon^2\rceil }\Big)$ 
 sets, step functions $W_{k}\in \mathcal{S}_{\cP_k}^2\cap \Y$ and $f_{k}\in \mathcal{S}_{\cP_k}^{1\to d}\cap \cL^{\infty}_r([0,1];\R^d)$, where $\Y$ is either $\cW_0$, $\cD_0$, $\cS_1$, or $\cW_1$, respectively such that
\[
  d_{\square}\big((W,f),(W_{k},f_k)\big)\leq \epsilon.
\]
\end{corollary}
\begin{proof}
First, invoke \cref{lem:gs-reg-lem0}, with errors $\norm{W-W_{k}}_{\square}\leq \nu$ and $\norm{f-f_{k}}_{\square} \leq \rho = \epsilon-\nu$. 
     We now show that there is some $\epsilon_0>0$  such that for every $\epsilon<\epsilon_0$, there is a choice of $\nu$ such that the number of sets in the partition, guaranteed by \cref{lem:gs-reg-lem0}, satisfies
     \[k(\nu):=\lceil r/(\epsilon-\nu) \rceil^d \Big( 2^{2\lceil 1/\nu^2\rceil }\Big) \leq  2^{\lceil 2c/\epsilon^2\rceil }.\]
 Denote $c=1+t$. 
 In case
 \begin{equation}
 \label{eq:how_large_nu}
   \nu \geq  \sqrt{\frac{2}{ 2(1+0.5t)/\epsilon^2-2}},  
 \end{equation}
 we have
\[2^{2\lceil 1/\nu^2\rceil } \leq 2^{ 2(1+0.5t)/\epsilon^2}.\]
On the other hand, for
\[\nu \leq \epsilon - \frac{r}{2^{t/(d\epsilon^2)}-1}, \]
 we have
\[\lceil r/(\epsilon-\nu) \rceil ^d\leq 2^{ 2(0.5t)/\epsilon^2}.\]
To reconcile these two conditions, restrict to $\epsilon$ such that
\begin{equation}\label{eq:epsilon_opt0}
    \epsilon - \frac{r}{2^{t/(d\epsilon^2)}-1} \geq \sqrt{\frac{2}{ 2(1+0.5t)/(\epsilon^2)-2}}.
\end{equation}
There exists $\epsilon_0$ that depends on $c$, $r$, and $d$ (and hence also on $t$) such that for every $\epsilon<\epsilon_0$, \cref{eq:epsilon_opt0} is satisfied.
Indeed, for small enough $\epsilon$,
\[\frac{1}{2^{t/(d\epsilon^2)}-1} = \frac{2^{-t/(d\epsilon^2)}}{1-2^{-t/(d\epsilon^2)}} \leq 2\cdot2^{-t/(d\epsilon^2)}< \frac{\epsilon}{r}\Big(1-\frac{1}{\sqrt{1+0.1t}}\Big),\]
so
\[\epsilon - \frac{r}{2^{t/(d\epsilon^2)}-1} > \frac{\epsilon}{\sqrt{1+0.1t}}.\]
Moreover, for small enough $\epsilon$,
\[\sqrt{\frac{2}{ 2(1+0.5t)/(\epsilon^2)-2}} = \epsilon\sqrt{\frac{1}{(1+0.5t)-(\epsilon^2)}}<\frac{\epsilon}{\sqrt{1+0.4t}}.\] 
Hence, for every $\epsilon<\epsilon_0$, there is a choice of $\nu$ such that
\[k(\nu) =\lceil r/(\epsilon-\nu) \rceil^d \Big( 2^{2\lceil 1/\nu^2\rceil }\Big)\leq  2^{ 2(0.5t)/(\epsilon^2)}2^{ 2(1+0.5t)/(\epsilon^2)} \leq
2^{\lceil 2c/\epsilon^2\rceil }.\]
Lastly, if necessary, split positive-measure parts of $\mathcal{P}_{k(\nu)}$ until the partition has exactly $k=2^{\lceil2c/\epsilon^2\rceil}$ parts. Lebesgue measure is atomless, and this refinement preserves the step functions and their approximation errors.
\end{proof}
\cref{lem:gs-reg-lem} extends \cite[Theorem B.8]{signal23} and follows the same proof strategy.
Below, brackets with an integer subscript, such as $[W^\phi]_n$ and $[f^\phi]_n$, denote chosen step approximants on $\cI_n$; brackets with a partition subscript denote the averaging projections from \cref{def:proj}.
\begin{theorem}[Regularity lemma for graphon-signals/kernel-signals -- equipartition version]
\label{lem:gs-reg-lem}
Let $c>1$ and $r>0$. For  any sufficiently small $\epsilon>0$, and every $(W,f) \in \X$,  where $\X$ is either $\WLrd$, $\DLrd$, $\SLrd$, or $\KLrd$, there exists $\phi\in S'_{[0,1]}$, step functions $[W^{\phi}]_{n}\in \mathcal{S}_{\cI_n}^2\cap \Y$ and $[f^{\phi}]_{n}\in \mathcal{S}_{\cI_n}^{1\to d}\cap \cL^{\infty}_r([0,1];\R^d)$, where $\Y$ is either $\cW_0$, $\cD_0$, $\cS_1$, $\cW_1$, respectively,
such that
\begin{equation}\label{eq:approximationProp}
  d_{\square}\Big(~(W^{\phi},f^{\phi})~,~\big([W^{\phi}]_{n},[f^{\phi}]_{n}\big)~\Big)\leq \epsilon,
\end{equation}
where $\cI_n$ is the equipartition of $[0,1]$ into $n= 2^{\lceil 2c/\epsilon^2\rceil}$ intervals.  
\end{theorem}
\begin{proof}
Let $c=1+t>1$, $\epsilon>0$ and $0<\alpha,\beta<1$.
In \cref{lem:gs-reg-lem00}, consider the approximation error 
\[
  d_{\square}\big((W,f),(W_{k},f_k)\big)\leq \alpha\epsilon.
\] 
with a partition $\cP_{k}$ into $k= 2^{\lceil \frac{2(1+t/2)}{(\epsilon\alpha)^2}\rceil}$ sets.
Next, equalize the partition $\cP_k$ up to error $\epsilon\beta$. More accurately, in \cref{l:equipartition}, choose 
\[n=\ \lceil 2^{ \frac{2(1+0.5t)}{(\epsilon\alpha)^2} + 1}/(\epsilon\beta)\rceil,\]
and note that
\[n \geq \frac{2^{\lceil \frac{2(1+0.5t)}{(\epsilon\alpha)^2}\rceil}}{\epsilon\beta} = \frac{k}{\epsilon\beta}.\]

By \cref{l:equipartition}, applied to $W_k$ with $p=2$ and scalar output, and to $f_k$ with $p=1$ and output dimension $d$, there is a common equipartition $\mathcal{E}_n$ into $n$ sets and corresponding step functions $W_n,f_n$. Since $\norm{W_k}_\infty\leq1$, $\norm{f_k}_\infty\leq r$, and $k/n\leq\epsilon\beta$, these satisfy
\[
\norm{W_k-W_n}_{\square}
\leq\norm{W_k-W_n}_1
\leq2\frac{k}{n}\leq2\epsilon\beta,
\qquad
\norm{f_k-f_n}_1
\leq dr\frac{k}{n}\leq dr\epsilon\beta.
\]
By the normalized signal cut norm,
\[
\norm{f_k-f_n}_{\square}
\leq\frac{1}{d}\norm{f_k-f_n}_1
\leq r\epsilon\beta.
\]
Hence, by the triangle inequality,
\[d_{\square}\big((W,f),(W_{n},f_n)\big) \leq d_{\square}\big((W,f),(W_{k},f_k)\big) + d_{\square}\big((W_k,f_k),(W_{n},f_n)\big) \leq \epsilon(\alpha+(2+r)\beta).\]
In the following, restrict the choices of $\alpha$ and $\beta$ to ones that satisfy $\alpha+(2+r)\beta=1$.
Consider the function  $n:(0,1)\rightarrow \NN$ defined by
\[n(\alpha):= \lceil 2^{ \frac{2(1+0.5t)}{(\epsilon\alpha)^2} + 1}/(\epsilon\beta)\rceil = \lceil (2+r)\cdot 2^{ \frac{2(1+0.5t)}{(\epsilon\alpha)^2} + 1}/(\epsilon(1-\alpha))\rceil.\]
Using a similar technique as in the proof of \cref{lem:gs-reg-lem00}, there is $\epsilon_0>0$ that depends on $c$ and $r$ (and hence also on $t$) such that for every  $\epsilon<\epsilon_0$ , we may choose $\alpha_0$ (that depends on $\epsilon$) which satisfies
\begin{equation}
    \label{eq:eqBound0}
    n(\alpha_0)
    =\lceil (2+r)\cdot 2^{ \frac{2(1+0.5t)}{(\epsilon\alpha_0)^2} + 1}/(\epsilon(1-\alpha_0))\rceil
    < 2^{\lceil \frac{2c}{\epsilon^2}\rceil}.
\end{equation}
Moreover, there is a choice  $\alpha_1$ which satisfies
\begin{equation}
    \label{eq:eqBound1}
    n(\alpha_1) 
    =\lceil (2+r)\cdot 2^{ \frac{2(1+0.5t)}{(\epsilon\alpha_1)^2} + 1}/(\epsilon(1-\alpha_1))\rceil
    > 2^{\lceil \frac{2c}{\epsilon^2}\rceil}.
\end{equation}
Note that the function $n:(0,1)\rightarrow \NN$
satisfies the following intermediate value property.  For  every $0<\alpha_1<\alpha_2<1$ and every $m\in \NN$ between $n(\alpha_1)$ and $n(\alpha_2)$, there is a point $\alpha\in[\alpha_1,\alpha_2]$ such that $n(\alpha)=m$. This follows the fact that $\alpha\mapsto (2+r)\cdot 2^{ \frac{2(1+0.5t)}{(\epsilon\alpha)^2} + 1}/(\epsilon(1-\alpha))$ is a continuous function. Hence, by \cref{eq:eqBound0} and \cref{eq:eqBound1}, there is a point $\alpha$ (and $\beta$ such that $\alpha+(2+r)\beta=1$) such that
\[ n(\alpha)=n= \lceil 2^{ \frac{2(1+0.5t)}{(\epsilon\alpha)^2} + 1}/(\epsilon\beta)\rceil = 2^{\lceil 2c/\epsilon^2\rceil}.\]
Choose $\phi\in S'_{[0,1]}$ mapping each interval of $\cI_n$ to the corresponding part of $\mathcal E_n$ up to null sets; such a map exists since the parts have equal measure. The $\cI_n$-step representatives of $W_n^\phi$ and $f_n^\phi$ then satisfy \cref{eq:approximationProp}, by invariance of the cut metric under simultaneous measure-preserving relabeling.
\end{proof}
Analogously to the derivation of \cite[Corollary~B.9]{signal23} from \cite[Corollary~B.8]{signal23}, a slight modification of the above proof allows us to replace \( n \) with the constant
\[
n = \left\lceil 2^{\frac{2c}{\epsilon^2}} \right\rceil .
\]
Consequently, for any \( n' \geq 2^{\lceil \frac{2c}{\epsilon^2} \rceil} \), the approximation property (\cref{eq:approximationProp}) holds with \( n' \) in place of \( n \). Indeed, keeping $c$ fixed, choose $\alpha_1\in(\alpha_0,1)$ sufficiently close to $1$ that $n(\alpha_1)>n'$. The same intermediate value argument gives $n(\alpha)=n'$ for some $\alpha\in(\alpha_0,\alpha_1)$, with the smallness condition on $\epsilon$ unchanged.
This yields the following corollary, which extends \cite[Corollary~B.9]{signal23}.
\begin{theorem} 
\label{lem:gs-reg-lem2}
Let $c>1$, $r>0$, and $d\in\NN$. For any sufficiently small $\epsilon>0$, for every $n \geq 2^{\lceil \frac{2c}{\epsilon^2}\rceil}$ and every $(W,f) \in \X$, where $\X$ is either $\WLrd$, $\DLrd$, $\SLrd$, or $\KLrd$, there exists $\phi\in S'_{[0,1]}$, step functions $[W^{\phi}]_{n}\in \mathcal{S}_{\cI_n}^{2\to1}\cap \Y$ and $[f^{\phi}]_{n}\in \mathcal{S}_{\cI_n}^{1\to d}\cap \cL^{\infty}_r([0,1];\R^d)$, where $\Y$ is either $\cW_0$, $\cD_0$, $\cS_1$, $\cW_1$, respectively, 
such that
\[
  d_{\square}\Big(~\big(W^{\phi},f^{\phi}\big)~,~\big([W^{\phi}]_{n},[f^{\phi}]_{n}\big)~\Big)\leq \epsilon,
\]
where $\cI_n$ is the equipartition of $[0,1]$ into $n$ intervals. 
\end{theorem} 
We now prove \cref{lem:attain}, which extends \cite[Lemma B.12]{signal23}, by following the same proof strategy.
\begin{lemma}\label{lem:attain}
Let $\mathcal{P}_n=\{P_1,\ldots,P_n\}$ be a partition of $[0,1]$, and
    Let $V,R\in \mathcal{S}_{\mathcal{P}_n}^{2\to1}\cap\Y$, where $\Y$ is either $\cW_0$, $\cD_0$, $\cS_1$, or $\cW_1$. Then, the supremum of
    \begin{equation}\label{eq:inf_piece0}
        \sup_{S,T\subset [0,1]} \abs{\int_S\int_T\big(V(x,y)-R(x,y)\big)dxdy}
    \end{equation}
    is attained for $S,T$ of the form
    \[S = \bigcup_{i\in s}P_i\ , \quad T = \bigcup_{j\in t}P_j,\]
    where $t,s\subset [n]$. Similarly for any two signals $f,g \in \mathcal{S}_{\mathcal{P}_n}^{1\to d}\cap \mathcal{L}_r^{\infty}([0,1];\R^d)$, the supremum of
\begin{equation}\label{eq:inf_piece1}
    \sup_{S\subset [0,1]} \norm{\int_S\big(f(x)-g(x)\big)dx}_1
\end{equation}
    is attained for $S$ of the form
    \[S = \bigcup_{i\in s}P_i,\]
    where $s\subset [n]$. 
\end{lemma}
\begin{proof}
    First, by \cref{lem:sup_attained},  
     the supremum of \cref{eq:inf_piece0} is attained for some $S,T\subset[0,1]$. If this supremum is zero, the kernel assertion follows with $S=T=\emptyset$. Otherwise, interchanging $V$ and $R$ if necessary, we may assume that
    \[ \int_S\int_T\big(V(x,y)-R(x,y)\big)dxdy >0.\]
    We can improve $T$ as follows. Let $t\subseteq[n]$ consist exactly of those indices $j$ for which
    \[ \int_S\int_{T\cap P_j}\big(V(x,y)-R(x,y)\big)dxdy >0.\]
    By increasing the set $T\cap P_j$ to $P_j$, we can only increase the size of the above integral. Indeed,
    \begin{align*}
      \int_S\int_{ P_j}\big(V(x,y)-R(x,y)\big)dxdy &   = \frac{\mu(P_j)}{\mu(T\cap P_j)}\int_S\int_{T\cap P_j}\big(V(x,y)-R(x,y)\big)dxdy\\
      & \geq\int_S\int_{T\cap P_j}\big(V(x,y)-R(x,y)\big)dxdy.
    \end{align*}
     Hence, discarding all nonpositive contributions and expanding the positive ones, we obtain
     \[T'=\bigcup_{j\in t}P_j\]
     and
     \[\int_S\int_{T'}\big(V(x,y)-R(x,y)\big)dxdy
     \geq\int_S\int_T\big(V(x,y)-R(x,y)\big)dxdy.\]
     We now repeat this process for $S$, which concludes the proof for the graphon part. 

For the signal case, set $h=f-g$. Since
$\norm{z}_1=\max_{\sigma\in\{-1,1\}^d}\sigma\cdot z$,
\[
\sup_{S\subseteq[0,1]}\norm{\int_S h(x)\,dx}_1
=\max_{\sigma\in\{-1,1\}^d}
\int_0^1\max\{\sigma\cdot h(x),0\}\,dx.
\]
For a maximizing sign vector $\sigma$, the set
$S=\{x:\sigma\cdot h(x)>0\}$ attains the supremum
and is a union of partition cells.
\end{proof}
Next, we show that the approximating graphon-signal can be constructed by taking, in each part, the average of the graphon and the signal. To this end, we define the projection of a graphon-signal onto a partition in \cref{sec:reg_lem}. For the reader’s convenience, we recall the definition below.
\proj*
\cref{lem:gs-reg-lem3} extends \cite[Corollary B.11]{signal23} and follows the same proof strategy. 
\theoremRegularity*
\begin{proof}
Set $\mathcal{P}_n:=\cI_n=\{P_1,\ldots,P_n\}$. Let $[W^\phi]_n\in \mathcal{S}_{\mathcal{P}_n}^{2\to1}\cap \Y$, where $\Y$ is either $\cW_0$, $\cD_0$, $\cS_1$, or $\cW_1$, and $[f^\phi]_{n}\in \mathcal{S}_{\cI_n}^{1\to d}\cap \mathcal{L}^{\infty}_r([0,1];\RR^d)$, be the step functions guaranteed by \cref{lem:gs-reg-lem2} with error $\epsilon/2$  and a measure preserving bijection $\phi\in S'_{[0,1]}$. We simplify notation and write $W_n:=[W^\phi]_{n}$ and $f_n:=[f^\phi]_{n}$. Until the final combined estimate, abbreviate $W^\phi,f^\phi$ by $W,f$, respectively. By \cref{lem:attain}, the supremum underlying $\norm{W_{\mathcal{P}_n} - W_n}_{\square}$ is attained for some
  \[S= \bigcup_{i\in s} P_i \ , \quad T= \bigcup_{j\in t} P_j.\]
  We now have, by definition of the projected graphon/kernel,
  \begin{align*}
     \norm{W_n - W_{\mathcal{P}_n}}_{\square} &   = \abs{\sum_{i\in s, j\in t}\int_{P_i}\int_{P_j} (W_{\mathcal{P}_n}(x,y) - W_n(x,y))dxdy}\\
     & =\abs{\sum_{i\in s, j\in t}\int_{P_i}\int_{P_j} (W(x,y) - W_n(x,y))dxdy}\\
     & = \abs{\int_{S}\int_{T} (W(x,y) - W_n(x,y))dxdy} \leq \norm{W_n - W}_{\square}.
  \end{align*}
  Hence, by the triangle inequality,
  \[\norm{W-W_{\mathcal{P}_n}}_{\square} \leq \norm{W-W_n}_{\square} + \norm{W_n-W_{\mathcal{P}_n}}_{\square} \leq 2\norm{W_n - W}_{\square}.\]
  A similar argument shows
  \[\norm{f-f_{\mathcal{P}_n}}_{\square} \leq 2\norm{f_n - f}_{\square}.\]
  Hence,
  \[ d_{\square}\Big(~\big(W^{\phi},f^{\phi}\big)~,~\big([W^{\phi}]_{\cI_n},[f^{\phi}]_{\cI_n}\big)~\Big)\leq 2d_{\square}\Big(~\big(W^{\phi},f^{\phi}\big)~,~\big([W^{\phi}]_{n},[f^{\phi}]_{n}\big)~\Big) \leq \epsilon.
\]
By the definition of the cut distance (\cref{eq:gs-metric}),
\[ 
\delta_{\square}\left(~\big(W,f\big)~,~\big((W,f)^\phi\big)_{\cI_n}~\right)\leq \epsilon.
\]
\end{proof}

\section{Compactness and Covering Numbers of Graphon-Signal and Kernel-Signal Spaces}\label{ap:compactness}
Here, we extend the compactness and covering results of \cite[Theorem 3.6]{signal23} to graphon-signals and kernel-signals with multidimensional signals.
The next theorem extends \cite[Theorem~C.1]{signal23}; its proof uses regularity and weak sequential compactness in a Hilbert space (see \cite{folland1999real}).
\begin{theorem}\label{theorem:cutcompact}
The metric spaces $\WLrdt$, $\DLrdt$, $\SLrdt$, and $\KLrdt$ are compact. 
\end{theorem}
\begin{proof}
    Consider a sequence $\{[(W_n,f_n)]\}_{n\in\NN}\subset \widetilde{\X}$, where $\widetilde{\X}$ is either $\WLrdt$, $\DLrdt$, $\SLrdt$, or $\KLrdt$, with $(W_n,f_n)\in\X$ , where $\X$ is $\WLrd$, $\DLrd$, $\SLrd$, or $\KLrd$, respectively.
     For each sufficiently large $k$,  consider the equipartition into $m_k$ intervals $\mathcal{I}_{m_k}$, where $m_k = 2^{30\lceil (r^2+1)\rceil k^2}$.  By \cref{lem:gs-reg-lem3} with $c=2$, there is a measure preserving bijection $\phi_{n,k}$ (up to nullset) such that
     \[\norm{(W_n,f_n)^{\phi_{n,k}} - (W_n,f_n)^{\phi_{n,k}}_{\mathcal{I}_{m_k}}}_{\square} \leq 1/k,\]
     where $(W_n,f_n)^{\phi_{n,k}}_{\mathcal{I}_{m_k}}$ is the projection of $(W_n,f_n)^{\phi_{n,k}}$ upon $\mathcal{I}_{m_k}$ (\cref{def:proj}).
     For every fixed $k$,  each pair of functions  $(W_n,f_n)^{\phi_{n,k}}_{\mathcal{I}_{m_k}}$ is determined by a coefficient vector in $[-1,1]^{m_k^2}\times[-r,r]^{dm_k}$. Hence, since $[-1,1]^{m_k^2}\times[-r,r]^{dm_k}$ is compact, there is a subsequence $\{n^k_j\}_{j\in\NN}$, such that all of these values converge. Namely, for each $k$, the sequence 
     \[\{(W_{n^k_j},f_{n^k_j})^{\phi_{n^k_j,k}}_{\mathcal{I}_{m_k}}\}_{j=1}^{\infty}\]
     converges pointwise to some step graphon-signal/kernel-signal $(U_k,g_k)$ in $[\X]_{\mathcal{I}_{m_k}}$ as $j\rightarrow\infty$. 
Choose the subsequences above nested in $k$ and take a diagonal subsequence. This diagonal sequence is $\delta_{\square}$-Cauchy: for each fixed sufficiently large $k$, the step approximants converge in $L^1$, and the triangle inequality bounds the pairwise cut distances by $2/k$ plus the cut metric between those approximants. Pass to a further subsequence $X_j=(W_{t_j},f_{t_j})$ with $\delta_{\square}(X_j,X_{j+1})<2^{-j}$. By relabeling invariance and the definition of $\delta_{\square}$, choose recursively $\psi_j\in S'_{[0,1]}$ such that $Y_j=X_j^{\psi_j}$ satisfies $d_{\square}(Y_j,Y_{j+1})<2^{-j}$. Thus $d_{\square}(Y_j,Y_\ell)\leq2^{1-j}$ for $\ell>j$.

The sequence $Y_j$ is bounded in the Hilbert space $\mathcal H=L^2([0,1]^2)\times(L^2([0,1]))^d$, so a subsequence $Y_{j_\ell}$ converges weakly to some $Y$. The admissible range and symmetry constraints are closed and convex, hence weakly closed, so we may choose a representative $Y\in\X$. Since the defining test integrals are weakly continuous, the combined cut norm is weakly lower semicontinuous. Consequently,
\[
\delta_{\square}(X_j,Y)
\leq d_{\square}(Y_j,Y)
\leq\liminf_{\ell\to\infty}d_{\square}(Y_j,Y_{j_\ell})
\leq2^{1-j}\longrightarrow0,
\]
which proves compactness.
     
\end{proof}

The equivalence relation $\sim$, is part of a class of equivalence relations called \emph{metric identifications}. A metric identification converts a pseudometric space into a metric space, while preserving the induced topologies. This identification is captured through the quotient map $\pi:\X\mapsto\widetilde{\X}$ defined as the map $x\to[x]$.
The open sets in the pseudometric space are exactly the sets of the form $\pi^{-1}(\A)$ where $\A$ is open in the quotient space. Namely, if an open set $\A$ in the pseudometric contains $x$, it has to contain all of the other elements in $[x]$. Using these facts, we prove \cref{corollary:cutcompact}.
\begin{theorem}
\label{corollary:cutcompact}
    The pseudometric spaces $(\WLrd,\delta_{\square})$, $(\DLrd,\delta_{\square})$, $(\SLrd,\delta_{\square})$, and $(\KLrd,\delta_{\square})$ and the metric spaces $(\WLrdt,\delta_{\square})$, $(\DLrdt,\delta_{\square})$, $(\SLrdt,\delta_{\square})$, and $(\KLrdt,\delta_{\square})$ are compact.
\end{theorem}
\begin{proof}
Let $\{U_\alpha\}_\alpha$ be an open cover of $\X$, where $\X$ denotes one of $\WLrd$, $\DLrd$, $\SLrd$, or $\KLrd$, and let $\widetilde{\X}$ denote the corresponding quotient space, that is, $\WLrdt$, $\DLrdt$, $\SLrdt$, or $\KLrdt$, respectively. Notice that $\pi^{-1}(\pi(U))=U$ for every open $U$ in the pseudometric topology, when $\pi:\X\mapsto\widetilde{X}$ is the quotient map. Moreover, every open set in the pseudometric topology is of the form $\pi^{-1}(V)$ where $V$ is open in the quotient topology. Thus, $\{\pi(U_\alpha)\}_\alpha$ is an open cover of $\widetilde{X}$. The compactness of $\widetilde{\X}$ gives us a finite cover $\{\pi(U_{\alpha_i})\}_i$. Using the mentioned facts, $\{U_{\alpha_i}\}_i$ is a finite cover of $\X$.
\end{proof}
We now extend \cite[Theorem C.2.]{signal23} following the same proof strategy.
\begin{theorem} 
\label{thm:cover}
Let $r>0$ and $c>1$. For every sufficiently small $\epsilon>0$ (with upper bound depending on $c,r,d$), each of the spaces $\WLrdt$, $\DLrdt$, $\SLrdt$, $\KLrdt$ can be covered by 
\begin{equation}
\kappa(\epsilon)=  2^{k^2}
\end{equation}
closed balls of radius $\epsilon$ in cut distance, where $k=\lceil 2^{2c/\epsilon^2}\rceil$.
\end{theorem}
\begin{proof}
Let $1<c<c'$ and $0<\alpha<1$.
Given an error tolerance  $\alpha\epsilon>0$, using \cref{lem:gs-reg-lem}, we take the equipartition $\cI_n$ into $n=2^{\lceil \frac{2c}{\alpha^2\epsilon^2}\rceil}$ intervals, for which any graphon-signal/kernel-signal $(W,f)\in\X$, where $\X$ is either $\WLrd$, $\DLrd$, $\SLrd$, or $\KLrd$ can, after a measure-preserving relabeling, be approximated by some $(W,f)_{n}$ in $[\X]_{\cI_n}$,  
 up to error $\alpha\epsilon$ in $d_{\square}$. Work with these relabeled representatives below. Consider the rectangle $\mathcal{R}_{n,r}=[-1,1]^{n^2}\times[-r,r]^{dn}$. We identify each element of $[\X]_{\cI_n}$ with an element of $\mathcal{R}_{n,r}$ using the coefficients of \cref{eq:Sd}. More accurately, the coefficients $c_{i,j}$ of the step graphon are identifies with the first $n^2$ entries of a point in $\mathcal{R}_{n,r}$, and the  the coefficients $b_{i}$ of the step signals are identifies with the last $dn$ entries of a point in $\mathcal{R}_{n,r}$.  
Set $h=(1-\alpha)\epsilon/2$ and define
\[
\tilde{\mathcal{R}}_{n,r}
=(h\mathbb{Z})^{n^2+dn}\cap\mathcal{R}_{n,r}.
\]
Its cardinality satisfies
\[
M:=|\tilde{\mathcal{R}}_{n,r}|
\leq(1+2/h)^{n^2}(1+2r/h)^{dn}.
\]
Round each scalar coefficient of $(W_n,f_n)=(W,f)_n$ toward zero to a multiple of $h$, and denote the resulting pair by $(W,f)_Q=(W_Q,f_Q)$. This preserves the kernel-signal class and changes each scalar coefficient by less than $h$. Only grid points corresponding to elements of $[\X]_{\cI_n}$ are used as centers. Hence,
\[
\norm{(W_n,f_n)-(W,f)_Q}_{\square}
\leq\norm{W_n-W_Q}_1+\frac1d\norm{f_n-f_Q}_1
<2h=(1-\alpha)\epsilon.
\]
We hence have
\[\norm{(W,f) - (W,f)_Q}_{\square} \leq \norm{(W,f) - (W_{n},f_{n})}_{\square} + \norm{(W_{n},f_{n})- (W,f)_Q}_{\square} \leq \epsilon.\]
Now fix $\alpha=\sqrt{2c/(c+c')}$ and let $k=\lceil2^{2c'/\epsilon^2}\rceil$. Since $n\leq2\cdot2^{(c+c')/\epsilon^2}$ and $c<c'$, the preceding count gives
\[
\log_2 M
\leq n^2\log_2(1+2/h)+dn\log_2(1+2r/h)
=o(k^2)
\qquad(\epsilon\downarrow0).
\]
Thus $M\leq2^{k^2}$ for sufficiently small $\epsilon$, with the threshold depending on $c,c',r,d$. The admissible grid points give the desired cover in $\delta_{\square}$, since relabeling does not change cut distance. For each prescribed $c'>1$, choose $c=(1+c')/2$ and then rename $c'$ as $c$ to obtain the stated bound.
\end{proof}
\cref{th:compactness} summarizes the above results.
\compactnesstheorem*

\section{Graphon-Signal and Kernel-Signal Sampling Lemmas}\label{ap:sampling}
In this appendix, we prove \cref{lem:second-sampling-garphon-signal00} by following the proof of \cite[Theorem 3.7]{signal23}. We include only minor necessary modifications to accommodate the multidimensional nature of the functions under consideration. We call any measurable functions $U:[0,1]^2\rightarrow[-1,1]$ a kernel and denote by $\mathcal{W}_1$ the space of all kernels. The following sampling lemmas are required for the proof of \cref{lem:second-sampling-garphon-signal00}.

The proof of \cref{cor:simple_sampled} follows the approach of \cite[Lemma 10.11]{lovasz2012large} and extends \cite[Corollary D.2]{signal23}.
 \begin{lemma}
\label{cor:simple_sampled} 
    Let $W\in\cD_0$ and $k\in\NN$. Then
    \[\mathbb{E}\big(d_{\square}(\mathbb{G}(W,\Lambda),W(\Lambda))\big) \leq \frac{11}{\sqrt{k}}.\]
 \end{lemma}
\begin{proof}
For $k\leq121$, the claim follows from $d_\square\leq1$. Assume $k>121$, fix $\varepsilon\geq10/\sqrt{k}$, and condition on $\Lambda$ throughout the following probability estimates.
For a directed weighted graph $G$ and $S,T\subseteq V(G)$, let $e_G(S,T)$ denote the sum of edge weights from $S$ to $T$. For $i,j \in [k]$, define the random variable $X_{ij} = \bm{1}_{(ij \in E(\mathbb{G}(W,\Lambda)))}$.
Let $S$ and $T$ be two disjoint subsets of $[k]$. Then the $X_{ij}$ ($i \in S, j \in T$) are independent,  
and $\mathbb{E}(X_{ij}) = w_{ij}$, where $w_{ij}$ is the corresponding weight of the $ij$ edge in $W(\Lambda)$, which gives that
\[
e_{\mathbb{G}(W,\Lambda)}(S,T) - e_{W(\Lambda)}(S,T) = \sum_{i \in S, j \in T} \big(X_{ij} - \mathbb{E}(X_{ij})\big).
\]

Let us call the pair $(S,T)$ \emph{bad}, if $|e_{\mathbb{G}(W,\Lambda)}(S,T) - e_{W(\Lambda)}(S,T)| > \varepsilon k^2 / 4$.  
For nonempty $S,T$, the probability of this can be estimated by the Chernoff--Hoeffding Inequality:
\[
\mathbb{P}\!\left(\left|\sum_{i \in S, j \in T} (X_{ij} - \mathbb{E}(X_{ij}))\right| > \tfrac{1}{4}\varepsilon k^2 \right)
 \leq 2\exp\!\left(-\frac{\varepsilon^2 k^4}{32 |S||T|}\right) 
 \leq 2\exp\!\left(-\frac{\varepsilon^2 k^2}{32}\right).
\]

The number of disjoint pairs $(S,T)$ is $3^k$, and so the probability that there is a bad pair is bounded by  
$2 \cdot 3^k e^{-\varepsilon^2 k^2 / 32} < e^{-\varepsilon^2 k^2 / 100}$.  
If there is no bad pair, then \cref{lem:attain} gives $d_\square(\mathbb{G}(W,\Lambda),W(\Lambda))\leq\varepsilon+1/k$. Indeed, independently coloring each vertex red or blue with equal probability expresses any off-diagonal cut sum as four times the expected disjoint cut sum between its red and blue parts; the diagonal contributes at most $1/k$.
Applying this inequality with $\varepsilon=10/\sqrt{k}$, bounding the distance by $1$ in the exceptional cases, and averaging over $\Lambda$, we obtain the following bound, since $1/k+e^{-k}\leq2/k\leq1/\sqrt{k}$:
\begin{equation*}
    \mathbb{E}\!\left( d_{\square}(\mathbb{G}(W,\Lambda),W(\Lambda)) \right) \leq \frac{11}{\sqrt{k}} .
\end{equation*}
This completes the proof.
\end{proof}

\begin{theorem}[\cite{signal23}, Corollary D.5, First sampling lemma - expected value version]
  \label{lem:first-sample}  
 Let $U \in \cW_1$ and $\Lambda\in [0,1]^k$ be chosen uniformly at random, where $k\geq 1$. Then
$$\mathbb{E}\abs{\|U[\Lambda]\|_{\Box} -\|U\|_{\Box}}\leq \frac{14}{k^{1/4}}.$$  
\end{theorem}
\begin{lemma}[\cite{signal23}, Lemma D.6., First sampling lemma for one dimensional signals]
 \label{cor:kernelsampling0}
 Let $f\in \mathcal{L}_r^{\infty}([0,1];\R)$. Then
 \[\mathbb{E}\abs{\norm{f(\Lambda)}_1 - \norm{f}_1} \leq \frac{r}{k^{1/2}}.\]
\end{lemma}
We now Generalize \cref{cor:kernelsampling0} to multidimensional signals.
\begin{lemma}[First sampling lemma for signals]
 \label{cor:kernelsampling}
 Let $f\in \mathcal{L}_r^{\infty}([0,1];\R^d)$. Then
 \[\mathbb{E}\abs{\norm{f(\Lambda)}_1 - \norm{f}_1} \leq \frac{dr}{k^{1/2}}.\]
\end{lemma}
\begin{proof}
For each channel of $f_i$  of $f$, the following holds from \cref{cor:kernelsampling0}.
\[\mathbb{E}\abs{\norm{f_i(\Lambda)}_1 - \norm{f_i}_1} \leq \frac{r}{k^{1/2}}.\]
Thus, by the triangle inequality and the definition of the $L_1([0,1];\R^d)$ norm, we get that
\[\mathbb{E}\abs{\norm{f(\Lambda)}_1 - \norm{f}_1} \leq \frac{dr}{k^{1/2}}.\]
\end{proof}
We prove \cref{lem:second-sampling-garphon-signal00,lem:second-sampling-garphon-signal02} together, by follows the proof of \cite[Theorem D.7]{signal23}.
\Samplinglemma*
\SamplinglemmaTwo*
\begin{proof}
Fix $r,d$ and, for sufficiently large $k$, set
\[
n=\left\lceil\frac{\sqrt{k}}{r\log(k)}\right\rceil,
\qquad \epsilon=\frac{2}{\sqrt{\log(n)}}.
\]
The construction in \cref{lem:gs-reg-lem2}, with $c=3/2$, supplies step approximants on $\mathcal I_n$ in the same kernel-signal class with the stronger error bound
\[
a+b:=\norm{W-W_n}_{\square}+\frac1d\norm{f-f_n}_1\leq\epsilon.
\]
Indeed, the midpoint quantization in \cref{lem:gs-reg-lem0} controls normalized signal $L^1$ error; common refinement preserves this bound, and the equipartition construction adds at most $2m/n$ to the kernel error and $rm/n$ to normalized signal $L^1$ error, where $m$ is the number of parts before equitization. Thus the same error budgets apply. The required condition $n\geq 2^{\lceil3/\epsilon^2\rceil}$ holds for large $k$. Here we have absorbed the fixed regularity relabeling into $(W,f)$: using the actual pullbacks preserves the joint sampling law, since the relabeled sample points are again independent and uniform.

Applying \cref{lem:first-sample} to $(W-W_n)/2$ and using the exact expectation of the sampled signal $L^1$ norm gives
\[
\mathbb E\norm{W(\Lambda)-W_n(\Lambda)}_{\square}
\leq a+28k^{-1/4},
\qquad
\frac1d\mathbb E\norm{f(\Lambda)-f_n(\Lambda)}_1=b.
\]
Let $N_i$ count the sample points in the $i$th interval of $\mathcal I_n$, and put
\[
p_i=\frac{N_i}{k},\qquad
\eta=\frac12\sum_{i=1}^n\left|p_i-\frac1n\right|.
\]
Since $N_i$ has the binomial distribution with parameters $k,1/n$,
\[
\mathbb E\eta\leq\frac12\sum_{i=1}^n
\sqrt{\frac{(1/n)(1-1/n)}{k}}
=\frac12\sqrt{\frac{n-1}{k}}.
\]
Choose one measure-preserving bijection $\phi$ matching mass $\min(1/n,p_i)$ of each original bin to sampled intervals carrying the same bin label. Such a map can be chosen measurably by subdividing $[0,1)$ into $nk$ equal intervals, matching $\min(k,nN_i)$ intervals of each label in increasing order, and matching the remaining intervals in increasing order. On a set of mass $1-\eta$, the signals agree and the kernels agree on its square. Since the kernels take values in $[-1,1]$ and the signal coordinates in $[-r,r]$,
\[
\norm{W_n-W_n(\Lambda)^\phi}_{\square}
+\frac1d\norm{f_n-f_n(\Lambda)^\phi}_1
\leq (4+2r)\eta.
\]
Using the same $\phi$ for both components, relabeling invariance, the triangle inequality, and $\norm{h}_{\square}\leq\norm{h}_1/d$ for signals, we obtain
\[
\mathbb E\delta_{\square}\big((W,f),(W(\Lambda),f(\Lambda))\big)
\leq 2(a+b)+(2+r)\sqrt{\frac nk}+28k^{-1/4}
\leq 2\epsilon+(2+r)\sqrt{\frac nk}+28k^{-1/4}.
\]
For graphon-signals, \cref{cor:simple_sampled} adds at most $11/\sqrt{k}$ when $W(\Lambda)$ is replaced by $\mathbb G(W,\Lambda)$.
Finally, $\log(n)/\log(k)\to1/2$, so multiplication of either upper bound by $\sqrt{\log(k)}$ gives the limit $4\sqrt2<15$. Both claims follow for all sufficiently large $k$, with a threshold depending only on $r,d$.
\end{proof}
\section{Equivalency of MPNNs on Graph-Signals, Graphon-Signals, and Kernel-Signals}
\label{Ap:Graphon-signal MPNNs}
In this appendix, we extend \cite[Lemma E.1]{signal23} to the setting of directed and undirected graphon-signals, as well as to symmetric and general kernel-signals with multidimensional signals. That is, applying an MPNN to a graph-signal and then inducing a graphon-signal/kernel-signal yields the same representation as first inducing the graphon-signal/kernel-signal and then applying the MPNN. Moreover, the same commutativity property holds for readout: applying readout directly to the graph-signal, or first inducing the graphon-signal/kernel-signal and then applying readout, produces the same result.

Given a kernel/graphon $W$ and a signal $f$, we define the \emph{message kernel} $\Phi_f:[0,1]^2\rightarrow\RR^p$ by \(\Phi_f(x,y) := \Phi(f(x),f(y)) = \sum_{k=1}^K \xi_{k,\text{rec}}(f(x))\odot\xi_{k,\text{trans}}(f(y)).\) The aggregation of a message kernel $Q :[0,1]^2\rightarrow\RR^p$, with respect to the kernel $W\in\mathcal{W}_1$, is defined to be 
\[\mathrm{Agg}(W,Q)(x) := \int_0^1 W(x,y) Q(x,y)dy.\]
\begin{lemma}
\label{MPL-graph-graphon}
Let $\Theta$ be a single layer MPNN without readout. For every graph-signal pair $(G,\mathbf{f})$
with node set $[n]$ and adjacency matrix $A = \{a_{ij}\}_{i,j \in [n]}$, 
the following holds:
\[\Theta(W_G,f_{\mathbf f})=f_{\Theta(G,\mathbf f)}
\quad\text{almost everywhere}.\]
 \end{lemma}
\begin{proof} 
Let $\{I_1,\ldots,I_n\}$ be the equipartition of $[0,1]$ to $n$ intervals. 
 For each $j\in[n]$, let $y_j\in I_j$ be an arbitrary point. Let $i\in[n]$ and  $x\in I_i$. We have
 \begin{align*}
     {\rm Agg}(G,\Phi_{\mathbf{f}})_i & = \frac{1}{n}\sum_{j\in[n]}a_{i,j}\Phi_{\mathbf{f}}(i,j)
     =\frac{1}{n}\sum_{j\in[n]}W_G(x,y_j)\Phi_{f_{\mathbf{f}}}(x,y_j)  \\
     & =\int_0^1 W_G(x,y)\Phi_{f_{\mathbf{f}}}(x,y)  dy={\rm Agg}(W_G,\Phi_{f_{\mathbf{f}}})(x). 
 \end{align*}
Therefore, for every $i\in[n]$ and every $x\in I_i$,
\begin{align*}
 f_{\Theta(G,\mathbf{f})}(x) &  =f_{\mu^{(1)}\big(\mathbf{f},{\rm Agg}(G,\Phi_{\mathbf{f}})\big)}(x)  =\mu^{(1)}\big(\mathbf{f}_i,{\rm Agg}(G,\Phi_{\mathbf{f}})_i\big)\\
 & = \mu^{(1)}\big(f_{\mathbf{f}}(x), {\rm Agg}(W_G,\Phi_{f_{\mathbf{f}}})(x)\big) = \Theta(W_G,f_{\mathbf{f}})(x).
\end{align*}
\end{proof}
\section{Lipschitz Continuity and Upper Bounds of MPNNs}\label{Ap:lip}
In this section we extend \cite[Theorem 4.1]{signal23}, which states that MPNNs are Lipschitz continuous with respect to the cut metric. The proofs largely follow the approach developed by Levie, but include necessary modifications to accommodate the multidimensional nature of the functions under consideration and the additional readout layer. The other Lipschitz continuity theorems in \cite{signal23} can be extended to multidimensional signals just as easily. The variation of the following lemmas to multidimensional signals is trivial.
\subsection{Lipschitz Continuity and Boundedness of Message Passing and Update Layers}
\label{p:message} Recall that, given a kernel/graphon \(W\) and a signal \(f\), the associated 
message kernel \(\Phi_f\) and its aggregation were defined in \cref{Ap:Graphon-signal MPNNs} as  
\(\Phi_f(x,y) = \Phi(f(x),f(y))\) and 
\(\mathrm{Agg}(W,\Phi_f)(x) = \int_0^1 W(x,y)\Phi_f(x,y)\,dy.\)
For message, update, and readout functions $F$, $L_F$ bounds both the $\ell^1\to\ell^1$ and the $\ell^\infty\to\ell^\infty$ Lipschitz constants. In the following lemma and its applications, $f,g\in\mathcal{L}^{\infty}_r([0,1];\R^d)$ and $\norm{\xi}_\infty:=\sup_{z\in[-r,r]^d}\norm{\xi(z)}_{\ell^\infty}$.
\begin{lemma}[\cite{signal23}, Lemma F.1., Product rule for message kernels]
\label{lem:message_dist}
Let $\Phi_f,\Phi_g$ be the message kernels corresponding to the signals $f,g$. Then  
\[\norm{\Phi_f-\Phi_g}_{1} \leq \sum_{k=1}^K\Big(L_{\xi_{k,{\rm rec}}}\norm{\xi_{k,{\rm trans}}}_{\infty} + \norm{\xi_{k, {\rm rec}}}_{\infty}L_{\xi_{k,{\rm trans}}} \Big) \norm{f- g}_1.\]
\end{lemma}
\begin{lemma}
\label{lem:message1}
Lemma F.2 and Lemma F.7 in \cite{signal23} are originally formulated for graphons. However, since the proof applies equally to kernels, we rephrase it accordingly.
[\cite{signal23}, Lemma F.2]
    Let $Q,V$ be two message kernels, and $W\in\cW_1$. Then
    \[\norm{\mathrm{Agg}(W,Q) - \mathrm{Agg}(W,V)}_1 \leq \norm{Q-V}_1. \]
\end{lemma}
\begin{lemma}[\cite{signal23}, Lemma F.7]
    \label{Lem:graphon_dist}
    Let $f\in \mathcal{L}_r^{\infty}([0,1],\R^d)$ , $W,V\in\cW_1$, and  suppose, for every $x\in[0,1]$ and $k=1,\ldots, K$, that $\norm{\xi_{k,{\rm rec}}(f(x))}_{\infty},\norm{\xi_{k,{\rm trans}}(f(x))}_{\infty}\leq \rho$. Then
    \[\norm{\mathrm{Agg}(W,\Phi_f) - \mathrm{Agg}(V,\Phi_f)}_{\square}   \leq 4K\rho^2\norm{W-V}_{\square}.\] 
     Moreover, if $\xi_{k,{\rm rec}}$ and $\xi_{k,{\rm trans}}$ are non-negatively valued for every $k=1,\ldots,K$, then
    \[\norm{\mathrm{Agg}(W,\Phi_f) - \mathrm{Agg}(V,\Phi_f)}_{\square} \leq  K\rho^2\norm{W-V}_{\square}.\]
\end{lemma}
Now, following the proof of \cite[Theorem F.8]{signal23}, we extend \cite[Corollary F.11]{signal23} to multidimensional signal.
\begin{corollary}
\label{cor:MPL_gen}
Suppose that for every ${\rm y}\in\{{\rm rec},{\rm trans}\}$ and $k=1,\ldots,K$
\[\norm{\xi_{k, y}(0)}_\infty \leq B, \quad L_{\xi_{k,y}} < L.\] 
     Then, for every $(W,f),(V,g)\in\KLrd$,
    \[\norm{{\rm Agg}(W, \Phi_f) - {\rm Agg}(V, \Phi_g)}_{\square}\leq 4Kd(L^2r + LB) \norm{f- g}_{\square} + 4K(Lr+B)^2\norm{W-V}_{\square}.\]
\end{corollary}
\begin{proof}
Let $(W,f),(V,g)\in\KLrd$. Averaging signs and using scalar positive and negative parts gives, for $h\in\mathcal{L}^1([0,1];\R^d)$,
\[
\int_0^1\norm{h(x)}_{\ell^\infty}\,dx
\leq 2^{-d}\sum_{\sigma\in\{-1,1\}^d}\int_0^1|\sigma\cdot h(x)|\,dx
\leq 2d\norm{h}_{\square}.
\]
Applying \cref{lem:message1,Lem:graphon_dist}, the coordinatewise product rule, and $\norm{q}_{\square}\leq\norm{q}_1/p$ for $p$-channel signals, we obtain
\begin{align*}
     \norm{{\rm Agg}(W, \Phi_f) - {\rm Agg}(V, \Phi_g)&}_{\square} 
     \leq \norm{{\rm Agg}(W, \Phi_f) - {\rm Agg}(W, \Phi_g)}_{\square} + \norm{{\rm Agg}(W, \Phi_g) - {\rm Agg}(V, \Phi_g)}_{\square} \\
    & \leq 2KL(Lr+B)\int_0^1\norm{f(x)-g(x)}_{\ell^\infty}\,dx + 4K(Lr+B)^2\norm{W-V}_{\square} \\
    & \leq 4Kd(L^2r + LB) \norm{f- g}_{\square} + 4K(Lr+B)^2\norm{W-V}_{\square}.
\end{align*}
\end{proof}
\begin{corollary}
\label{cor:MPL_gen2}
Let $p$ be the message dimension. Suppose that for every $y\in\{{\rm rec},{\rm trans}\}$ and $k=1,\ldots,K$
\[\norm{\mu(0,0)}_\infty,\ \norm{\xi_{{k,y}}(0)}_\infty \leq B, \quad L_{\mu},\  L_{\xi_{k,y}} < L.\] 
     Then, for every $(W,f),(V,g)\in\KLrd$,
\begin{align*}
\norm{\mu\left(f,{\rm Agg}(W, \Phi_f)%
\right) &- \mu \left(g,{\rm Agg}(V, \Phi_g)%
\right)}_{\square}\\&\leq (8Kdp(L^3r + L^2B)+2dL) \norm{f- g}_{\square} + 8KpL(Lr+B)^2\norm{W-V}_{\square}.
\end{align*}
\end{corollary}
\begin{proof}
Let $(W,f),(V,g)\in\KLrd$. Using the norm comparison in the preceding proof, $\norm{q}_{\square}\leq\int_0^1\norm{q(x)}_{\ell^\infty}\,dx$, the $\ell^\infty$-Lipschitz bound for $\mu$, and \cref{cor:MPL_gen}, we obtain
\begin{align*}
&\norm{\mu(f,{\rm Agg}(W,\Phi_f))-\mu(g,{\rm Agg}(V,\Phi_g))}_{\square}\\
&\leq 2L\Big(d\norm{f-g}_{\square}+p\norm{{\rm Agg}(W,\Phi_f)-{\rm Agg}(V,\Phi_g)}_{\square}\Big)\\
&\leq (8Kdp(L^3r+L^2B)+2dL)\norm{f-g}_{\square}
+8KpL(Lr+B)^2\norm{W-V}_{\square}.
\end{align*}
\end{proof}
\subsection{Lipschitz Continuity and Boundedness of MPNNs} 
We now present \cref{lem:rec}, which will be used to compute Lipschitz constants for MPNNs.
\begin{lemma}[\cite{signal23}, Lemma F.16]
\label{lem:rec}
 Let $\mathbf{a}=(a_0,a_1,\ldots)$ and $\mathbf{b}=(b_0,b_1,\ldots)$.   The solution to
    $e_{t+1}= a_te_t+b_t$, with initialization $e_0$, is, for $t\geq1$,
    \begin{equation*}
        e_{t} = Z_t(\mathbf{a},\mathbf{b},e_0):=\prod_{j=0}^{t-1} a_je_0 + \sum_{j=1}^{t} \prod_{i=1}^{j-1}a_{t-i}b_{t-j},
    \end{equation*}
    where, by convention,
    \[\prod_{i=1}^{0}a_{t-i}:=1.\]
    In case there exist $a,b\in\RR$ such that $a_i=a$ and $b_i=b$ for every $i$,
     \[e_{t} = a^t e_0 + \sum_{j=0}^{t-1} a^jb.\]
\end{lemma}

Next, we extend \cite[Lemma F.10]{signal23} to the multidimensional setting and provide an additional bound for signal aggregation followed by the application of a Lipschitz continuous update function.
\begin{lemma}
\label{lem:MPL_infty}
Let $(W,f)\in\KLrd$, and suppose that for every $ y\in\{{\rm rec},{\rm trans}\}$ and $k=1,\ldots,K$
\[\norm{\mu(0,0)}_\infty,\ \norm{\xi_{k,y}(0)}_\infty \leq B, \quad L_{\mu}, L_{\xi_{k,y}} < L.\] 
     Then, 
\[\norm{\xi_{k,y}\circ f}_{\infty} \leq  Lr + B,\]
\[\norm{{\rm Agg}(W, \Phi_f)}_{\infty} \leq {K(Lr + B)^2}.\]
and
\[\norm{\mu(f,{\rm Agg}(W, \Phi_f)%
)}_{\infty} \leq  Lr+KL(Lr+B)^2+B.\]
\end{lemma}
\begin{proof}
   We denote the zero vector by $0$. Let $y\in\{{\rm rec},{\rm trans}\}$. 
    We have
    \[\norm{\xi_{k,y}(f(x))}_{\infty} \leq \norm{\xi_{k, y}(f(x)) - \xi_{k, y}(0)}_{\infty} + B \leq L_{\xi_{k, y}}\norm{f(x)}_\infty + B \leq Lr + B,\]
    so, 
    \begin{align*}
        \norm{{\rm Agg}(W, \Phi_f)(x)}_{\infty} & = \norm{\sum_{k=1}^K \int_0^1 W(x,y)\big(\xi_{k,{\rm rec}}(f(x))\odot\xi_{k,{\rm trans}}(f(y))\big)dy}_\infty \leq K(Lr + B)^2.    \end{align*}
and
\begin{align*}
\norm{\mu(f,{\rm Agg}(W, \Phi_f)%
)}_{\infty} &\leq \norm{\mu(f,{\rm Agg}(W, \Phi_f)%
) - \mu(0,0)}_{\infty}+B\\&\leq L(\norm{f}_\infty+\norm{{\rm Agg}(W, \Phi_f)}_{\infty})%
+B\\&\leq Lr+KL(Lr+B)^2+B.
\end{align*}
\end{proof}

The following lemma gives an upper bound for the output of MPNNs on graphon-signals. The proof follows the proof strategy of \cite[Lemma F.18]{signal23}.
\begin{lemma}\label{lem:MPNNsbound}
    Let $\Theta$ be a MPNN with $T$ layers. 
    Suppose that for every layer $t$ and every $y\in\{{\rm rec},{\rm trans}\}$ and $k\in[K]$,\[\norm{\mu^{(t)}(0,0)}_\infty,~\norm{\xi^{(t)}_{k,y}(0)}_\infty \leq B, \quad L_{\mu^{(t)}},\ L_{\xi^{(t)}_{k,y}} < L\]
    with $K\geq1$ and $L,B>1$. Let $(W,f)\in\KLrd$.
    Then, for MPNN with update function, for every layer $t$,
 \[\norm{\Theta_t(W,f)}_{\infty} \leq (6KL^3B^2)^{2^t}\max\{1,\norm{f}_{\infty}\}^{2^t},\]
\end{lemma}
\begin{proof}
Set $C_0:=\max\{1,\norm{f}_{\infty}\}$ and $a:=6KL^3B^2$, and define $C_{t+1}:=aC_t^2$.
Since $C_t\geq1$, $K\geq1$ and $L,B>1$, we have
\[
LC_t+KL(LC_t+B)^2+B\leq 6KL^3B^2C_t^2=C_{t+1}.
\]
Thus \cref{lem:MPL_infty} implies inductively that $\norm{\Theta_t(W,f)}_{\infty}\leq C_t$.
Moreover,
\[
C_t=a^{2^t-1}C_0^{2^t}\leq a^{2^t}C_0^{2^t},
\]
which proves the claim.
\end{proof}
We now extend \cite[Corollary F.19]{signal23} to the multidimensional setting for MPNNs with update functions, both with and without readout, by following the proof provided in \cite{signal23}.
\begin{theorem}\label{th:lipwithconst}
Let $\Theta$ be an MPNN with $T$ layers and node dimensions $d_0=d,d_1,\ldots,d_T$, and let $p_j$ be the message dimension for the transition from layer $j$ to layer $j+1$. Let $\psi$ be a Lipschitz readout function.
Suppose that for every layer $t$, every $y\in\{{\rm rec},{\rm trans}\}$ and every $k\in[K]$,
\[\norm{\mu^{(t)}(0,0)}_\infty,\ \norm{\xi^{(t)}_{k,y}(0)}_\infty,\ \norm{\psi(0)}_\infty\leq B,\quad
L_{\mu^{(t)}},\ L_{\xi^{(t)}_{k,y}},\ L_\psi<L,\]
with $K\geq1$ and $L,B>1$. Let $(W,f),(V,g)\in\KLrd$.
For $0\leq j\leq T$, put
\[
r_j:=(6KL^3B^2)^{2^j}\max\{1,r\}^{2^j},
\]
and, for $0\leq j<T$, define
\[
a_j:=8Kd_jp_j(L^3r_j+L^2B)+2d_jL,
\qquad b_j:=8Kp_jL(Lr_j+B)^2.
\]
For $1\leq t\leq T$, let
\[
E_t:=\left(\prod_{j=0}^{t-1}a_j\right)\norm{f-g}_{\square}
+\sum_{j=1}^{t}\left(\prod_{i=1}^{j-1}a_{t-i}\right)b_{t-j}\norm{W-V}_{\square},
\]
where an empty product equals $1$. Then
\[
\norm{\Theta_t(W,f)-\Theta_t(V,g)}_{\square}\leq E_t,
\qquad 1\leq t\leq T,
\]
and, with readout,
\[
\|\Theta(W,f)-\Theta(V,g)\|_\infty\leq Ld_T E_T.
\]
\end{theorem}
\begin{proof}
Both input signals are bounded by $r_0$, and \cref{lem:MPNNsbound} bounds both feature sequences at each layer $j\geq1$ by $r_j$.
Set $D_0:=\norm{f-g}_{\square}$ and
$D_t:=\norm{\Theta_t(W,f)-\Theta_t(V,g)}_{\square}$ for $1\leq t\leq T$.
By \cref{cor:MPL_gen2},
\[
D_{j+1}\leq a_jD_j+b_j\norm{W-V}_{\square},
\qquad 0\leq j<T.
\]
Since $a_j>0$, applying \cref{lem:rec} to the corresponding equality recurrence gives $D_t\leq E_t$.
For the readout,
\begin{align*}
&\|\Theta(W,f)-\Theta(V,g)\|_\infty\\
&\quad\leq L\left\|\int_0^1
\big(\Theta_T(W,f)-\Theta_T(V,g)\big)\,dx\right\|_{\ell^\infty}\\
&\quad\leq Ld_TD_T\leq Ld_TE_T.
\end{align*}
\end{proof}
\begin{remark}
By \cref{claim:cutProp}, the Lipschitz continuity in \cref{th:lipwithconst} also holds when the signal cut norms are replaced by the product cut norm $\norm{\cdot}_{\square_\times}$, with constants depending on the feature dimensions.
\end{remark}
Apply \cref{lem:MPNNsbound,th:lipwithconst} with $\widehat L:=1+L$ and $\widehat B:=1+B$. Writing the resulting bound as $E_t=A_t\norm{f-g}_\square+Q_t\norm{W-V}_\square$, we may take
\[
L_{\mathcal{H}_t}:=A_t+Q_t,\qquad
L_{\mathcal{H}}:=\widehat Ld_T(A_T+Q_T),\qquad
B_{\mathcal{H}}:=\widehat B+\widehat Lr_T.
\]
These constants are uniform over networks with the specified dimensions. Each layer commutes with measure-preserving relabelings, so $\Theta_t((V,g)^\phi)=\Theta_t(V,g)^\phi$ almost everywhere, and the mean readout is invariant under such relabelings. Apply the aligned estimates to $(W,f)$ and $(V,g)^\phi$ and take the infimum over $\phi\in S'_{[0,1]}$ to obtain the two cut-distance inequalities. They also show that the node-feature maps, modulo zero signal distance, and the readout are well defined on $\X$. The output bound follows from $\norm{\psi(z)}_\infty\leq B+L\norm{z}_\infty$ and the radius $r_T$. For fixed depth and dimensions, the displayed constants have polynomial bounds in $K,B,L,\max\{1,r\}$.
\Lipschitness*
\subsection{Lipschitz Continuity of Simplified MPNNs on Graphons and Kernels with Single Dimensional Signals}
Here, we state a Lipschitz continuity theorem for MPNNs without message functions, i.e.,  
$\xi_r(x) = 1$ and $\xi_t(x) = x$ for $x \in \R$ only considering the aggregated information and not the previous node's signal value. %
At each layer $t$, the MPNN processes an input signal \(f^{(t-1)}\) into the output signal  
\[
f^{(t)}(x):=\mu^{(t)}\!\left(\mathrm{Agg}(W,f)\right), \quad {\rm when}\quad \mathrm{Agg}(W,f):=\int^1_{0} W(y,x) f^{(t-1)}_y \, \mathrm{d}y.
\]
The proof of the following variation of \cref{Lem:graphon_dist} follows similar steps to the proof of \cite[Lemma F.7]{signal23}.
\begin{lemma}\label{lem:AggBound}
Let $f\in \mathcal{L}_r^{\infty}([0,1],\R)$ , $W,V\in\cW_1$. Then
    \[\norm{\mathrm{Agg}(W,f) - \mathrm{Agg}(V,f)}_{\square}   \leq 2r\norm{W-V}_{\square}.\] 
Moreover, if $f$ is non-negatively valued, then 
\[\norm{\mathrm{Agg}(W,f) - \mathrm{Agg}(V,f)}_{\square}   \leq r\norm{W-V}_{\square}.\] 
\end{lemma}
\begin{proof}
Using $f_\pm(y)=\int_0^r\mathbf{1}_{\{f_\pm>t\}}(y)\,dt$ and Fubini's theorem, each positive-part term below is at most $r\norm{W-V}_\square$. If $f\ge0$, the second term vanishes.
\begin{align*}
\abs{\int_{S}\mathrm{Agg}(W,f)(x) - \mathrm{Agg}(V,f)(x) dx} =&\abs{\int_S\int_0^1(W-V)(y,x)f(y)dydx}\\\leq&\abs{\int_S\int_0^1(W-V)(y,x)f_+(y)dydx}\\&+\abs{\int_S\int_0^1(W-V)(y,x)f_-(y)dydx}\\\leq&2r\norm{W-V}_{\square}
\end{align*}
\end{proof}
We now use \cref{lem:AggBound} to prove the following theorem.
\begin{lemma}
\label{cor:SimpMPL_gen}
Suppose that \( L_{\mu} < 1\). Then, for every $(W,f),(V,g)\in\KLr$,
    \[\norm{\mu({\rm Agg}(W,f))-\mu({\rm Agg}(V,g))}_{\square}\leq2\norm{f-g}_\square+4r\norm{W-V}_\square.\]
Moreover, if $f$ (or $g$) is non-negatively valued, then 
    \[\norm{\mu({\rm Agg}(W,f))-\mu({\rm Agg}(V,g))}_{\square}\leq2\norm{f-g}_\square+2r\norm{W-V}_\square.\]
\end{lemma}
\begin{proof}
For every integrable scalar signal $h$,
$\norm{h}_{\square}\le\norm{h}_1\le2\norm{h}_{\square}$.
Since $|W|\le1$, Fubini's theorem also gives
$\norm{{\rm Agg}(W,h)}_1\le\norm{h}_1$.
Thus, using $L_\mu<1$ and \cref{lem:AggBound},
\begin{align*}
&\norm{\mu({\rm Agg}(W,f))-\mu({\rm Agg}(V,g))}_{\square}\\
&\quad\le\norm{{\rm Agg}(W,f-g)}_1
       +\norm{{\rm Agg}(W,g)-{\rm Agg}(V,g)}_1\\
&\quad\le2\norm{f-g}_{\square}
       +2\norm{{\rm Agg}(W,g)-{\rm Agg}(V,g)}_{\square}\\
&\quad\le2\norm{f-g}_{\square}
       +4r\norm{W-V}_{\square}.
\end{align*}
If $g\ge0$, \cref{lem:AggBound} improves the last coefficient to $2r$.
If $f\ge0$, interchange $(W,f)$ and $(V,g)$.
\end{proof}
\cref{lem:SimpMPNNsbound} establishes an upper bound on the signals processed by simplified MPNNs.
\begin{lemma}
\label{lem:SimpMPNNsbound}
Let $\Theta$ be a MPNN with $T$ layers.  Suppose that for every layer $t$,, $\abs{\mu^t(0)}=0$ and $L_{\mu^t} < 1$. Let $(W,f)\in\KLr$, then, for MPNN with or without update function, for every layer $t$,
    \[\norm{\Theta_t(W,f)}_{\infty} \leq r.\]
\end{lemma}
\begin{proof}
Let $(W,f)\in\KLr$ and suppose that $\abs{\mu(0)}=0$ and $L_{\mu}\leq1$. Then,
\begin{align*}
\abs{\mu({\rm Agg}(W, f))} &\leq \abs{\mu({\rm Agg}(W, f)) - \mu(0)}+\abs{\mu(0)}\leq\abs{{\rm Agg}(W,f)} = \abs{ \int_0^1W(y,x)f(y)dy}\leq r.
\end{align*}
Iterating this bound proves the claim; without update functions, use the same aggregation bound.
\end{proof}
We now prove the Lipschitz property of simplified MPNNs without a readout. To establish Lipschitz continuity with a readout, follow the steps in \cref{th:lipwithconst}.
\begin{theorem}\label{th:simplipwithconst}
Let $\Theta$ be an MPNN without message functions, with $T$ layers and $\psi$ be a Lipschitz readout function, that updats kernel-signal node features only by performing aggregation over the in-neighbours $(j \rightarrow i)$. Suppose that for every layer $t$, 
\[\abs{\mu^{(t)}(0)} = 0,\quad L_{\mu^{(t)}} < 1.\]
Let $(W,f),(V,g)\in\KLr$. Then, for MPNNs with and without update functions and every $t\in[T]$,
\begin{align*}
   \norm{\Theta_{t}(W, f) - \Theta_{t}(V, g)}_{\square} \leq  & ~2\norm{f-g}_{\square}+4rt\norm{W-V}_{\square}.
\end{align*}
In addition, if $f$ and $g$ are non-negative valued functions, then
\begin{align*}
   \norm{\Theta_{t}(W, f) - \Theta_{t}(V, g)}_{\square} \leq  & ~2\norm{f-g}_{\square}+(4t-2)r\norm{W-V}_{\square}.
\end{align*}
\end{theorem}
\begin{proof}
Write $f_j=\Theta_j(W,f)$, $g_j=\Theta_j(V,g)$, $f_0=f$, $g_0=g$, and $e_j=\norm{f_j-g_j}_1$.
By \cref{lem:SimpMPNNsbound}, both feature sequences are bounded in absolute value by $r$.
The updates are nonexpansive, and aggregation by $W$ is a contraction in $L^1$. Thus, for $0\le j<T$,
\begin{align*}
e_{j+1}
&\le e_j+\norm{{\rm Agg}(W,g_j)-{\rm Agg}(V,g_j)}_1\\
&\le e_j+2\norm{{\rm Agg}(W,g_j)-{\rm Agg}(V,g_j)}_{\square}\\
&\le e_j+4r\norm{W-V}_{\square},
\end{align*}
where the last inequality follows from \cref{lem:AggBound}. The same estimate holds without updates.
Since $e_0\le2\norm{f-g}_{\square}$, induction and $\norm{f_t-g_t}_{\square}\le e_t$ prove the first bound.
If $f,g\ge0$, the first increment is at most $2r\norm{W-V}_{\square}$; the remaining $t-1$ increments are at most $4r\norm{W-V}_{\square}$ each, giving the second bound.
\end{proof}
\section{A Generalization Theorem for MPNNs}\label{AP:gen}
In this appendix we prove \cref{theorem:generalization}. We start by presenting the learning problem setting. 

\subsection{Statistical Learning and Generalization Analysis}
In statistical learning theory we consider a sample space $\mathcal{P}=\X\times\Y$, that represents all possible data. We name any arbitrary probability measure on $(\mathcal{P}, \mathcal{B}(\mathcal{P}))$ a \emph{data distribution}. Assume we have some fixed and unknown data distribution $\nu$. We may assume that we complete $\mathcal{B}(\mathcal{P})$ with respect to $\nu$ to a complete $\sigma$-algebra $\Sigma$ or just denote $\Sigma=\mathcal{B}(\mathcal{P})$, Since the completeness of our measure space does not affect our construction. Let $\mathbf{X}=(X_i)_{i=1}^N\in\mathcal{P}^N$ be a dataset of independent random samples from $(\mathcal{P},\Sigma, \nu)$. We assume $\Y$ and $\X$ relate according to some fixed and unknown regular conditional distribution $\nu_{\Y|\X}(x,\cdot)\in\mathscr{P}(\Y)$, for $x\in\X$. The learning problem is choosing the function that best approximate the relation between the points in $\Y$ and the points in $\X$ from some set of functions. 

Let $\mathcal{L}$ be a loss function, i.e., a non-negative and upper-bounded uniformly by a scalar $M$. Note that the loss $\mathcal{L}$ can have a learnable component that depends on the dataset $\mathbf{X}$. All hypotheses and losses are assumed measurable, and population risk is evaluated on an independent draw from $\nu$ with the fitted model and loss held fixed. Our goal is to find an optimal model $\Theta$ from some \emph{hypothesis space} $\mathcal{T}$ that has a low \emph{statistical risk}
\begin{equation*}
 \mathcal{R}(\Theta)= \mathbb{E}_{(x,y) \sim \nu}[\mathcal{L}(\Theta(x), y)] = \int\mathcal{L}(\Theta(x),y) d\nu(x,y), \quad \Theta \in \mathcal{T}.
\end{equation*}
However, the true distribution $\nu$ is not directly observable. Instead, we can access a set of independent, identically distributed samples $\mathbf{X} = (X_1, \ldots, X_N)$ from $(\mathcal{P},\Sigma, \nu)$ where $0<i\leq N:X_i=(Q_i,Y_i)\in\mathcal{P}=\X\times\Y$. Instead of minimizing the statistical risk with some unknown data distribution $\nu$, we try to approximation the optimal model by minimizing the \emph{empirical risk}:
\[
\hat{\mathcal{R}}_\mathbf{X}(\Theta_{\mathbf{X}}) = \frac{1}{N} \sum_{i=1}^N  \mathcal{L}(\Theta_\mathbf{X}(Q_i),Y_i),
\]
where $\Theta_{\mathbf{X}}$ is a model, that might depend on the sampled data, e.g., through training. The goal of generalization analysis is to show that a low empirical risk of a network entails a low statistical risk as well. One common approach is to bound the statistical risk involves using the inequality:

\[
\mathcal{R}(\Theta) \leq \hat{\mathcal{R}}(\Theta) + E,
\]
where $E$, the generalization error, is defined as:

\[
E = \sup_{\Theta \in \mathcal{H}} |\mathcal{R}(\Theta) - \hat{\mathcal{R}}(\Theta)|.
\]

Note that $\Theta := \Theta_\mathbf{X}$ depends $\mathbf{X}$ since it means that the empirical risk is not truly a Monte Carlo approximation of the statistical risk in the learning context, since the network is not constant when varying the dataset. If the model $\Theta$ and the loss $\mathcal{L}$ were fixed independently of $\mathbf{X}$, Monte Carlo theory would give us an order $\mathcal{O}(\sqrt{\rho(p)/N})$ bound for $|\mathcal{R}(\Theta)-\hat{\mathcal{R}}(\Theta)|$ with probability $1 - p$, where $\rho(p)$ depends on the specific inequality used, e.g., $\rho(p) = \log(2/p)$ in Hoeffding's inequality.

We call such events \emph{good sampling events}. Their dependence on the model $\Theta$ results in the requirement of intersecting all good sampling events in $\mathcal{T}$, in order to compute a naive
bound to the generalization error. Uniform convergence bounds are used to intersect appropriate sampling events. Thus, they allow for more efficient bounds of the generalization error. This intersection presents the terms \emph{complexity} and \emph{capacity}, which describe the richness of the hypothesis space $\mathcal{Z}$ and underly different approaches. Some of which are VC-dimension, Rademacher dimension, fat-shattering dimension, pseudo-dimension, and uniform covering number (see, e.g., \cite{statBackground}).
\subsection{Robustness and Generalization}
\label{Ap:Uniform Monte Carlo approximation of Lipschitz continuous functions}
We begin by proving \cref{theorem:robustGen} using the Bretagnolle–Huber–Carol inequality (\cref{proposition:BHCinequality}). For the same metric probability space at a fixed covering radius and confidence level, \cref{theorem:robustGen} improves the dependence on the covering number in the uniform Monte Carlo estimate used in the proof of \cite[Theorem G.3]{signal23}. This result serves as the foundation for the proof of \cref{theorem:generalization}. To set the stage, we first introduce the Bretagnolle–Huber–Carol inequality.
\begin{proposition}[\cite{Vaart2000}, Proposition A.6.6, Bretagnolle–Huber-Carol Inequality]\label{proposition:BHCinequality}
Let $n,N\ge1$ be integers, let $(\X,\Sigma,\nu)$ be a probability space, and let $(N_1,\dots,N_N):\X\to\{0,\dots,n\}^N$ have a multinomial distribution with parameters $n$ and $(p_1,\dots,p_N)$. Then
\[
\nu\left( \sum_{i=1}^{N} |N_i - n p_i| \geq 2 \sqrt{n} \lambda \right) \leq 2^N \exp(-2\lambda^2), \quad \lambda > 0.
\]
\end{proposition}
Recall that the \( L_\infty \) norm of a function \( \Upsilon: \X \to \R \) is defined as
\(
\|\Upsilon\|_{\infty} = \sup_{x \in \X} |\Upsilon(x)|.
\).
We prove \cref{theorem:robustGen} by following the proofs of \cite[Theorem 3]{Xu2012} and \cite[Theorem 4.2]{signal23}.
\begin{theorem}\label{theorem:robustGen}
Let $(\mathcal{P},\rho)$ be a nonempty totally bounded metric space, let $\Sigma$ contain its Borel $\sigma$-algebra, and let $\nu$ be a probability measure on $\Sigma$. For $r>0$, let $\kappa(r)$ be the minimum number of closed balls of radius $r$, centered in $\mathcal{P}$, needed to cover $\mathcal{P}$. Let $N\ge1$ be an integer and let $\mathbf{X}=(X_1,\ldots,X_N)$ be drawn i.i.d. from $\nu$.
For every fixed $r>0$ and $0<p\le1$, there exists an event $\mathcal{E}^{p,r}_{\rm Lip}\in\Sigma^{\otimes N}$ with
\[
\nu^N(\mathcal{E}^{p,r}_{\rm Lip})\ge1-p
\]
such that, on this event, simultaneously for every bounded Lipschitz function $F:\mathcal{P}\to\mathbb{R}^d$,
\[
\left\|\int_{\mathcal{P}}F(x)\,d\nu(x)
-\frac1N\sum_{i=1}^N F(X_i)\right\|_{\ell^\infty}
\le 2rL_F+\|F\|_\infty
\sqrt{\frac{2\bigl(\kappa(r)\ln2+\ln(1/p)\bigr)}{N}}.
\]
Here $L_F$ is a Lipschitz constant with respect to $\rho$ and the target norm $\|\cdot\|_{\ell^\infty}$, and $\|F\|_\infty:=\sup_{x\in\mathcal{P}}\|F(x)\|_{\ell^\infty}$. The radius $r$ is chosen independently of the sample.
\end{theorem}
\begin{proof}
Fix $r>0$ and choose a cover by $J=\kappa(r)$ balls $B_1,\ldots,B_J$ of radius $r$, with centers $z_1,\ldots,z_J\in\mathcal{P}$. Set $I_1=B_1$ and $I_j=B_j\setminus\bigcup_{i<j}B_i$ for $j\ge2$. These sets form a measurable partition of $\mathcal{P}$.
The counts $S_j:=\sum_{i=1}^N\mathbf{1}_{I_j}(X_i)$ have a multinomial distribution with parameters $N$ and $(\nu(I_1),\ldots,\nu(I_J))$.
By \cref{proposition:BHCinequality}, the measurable event $\mathcal{E}^{p,r}_{\rm Lip}$ defined by
\begin{equation}\label{equ.proofmain}
\sum_{j=1}^J\left|\nu(I_j)-\frac{S_j}{N}\right|
\le\sqrt{\frac{2\bigl(J\ln2+\ln(1/p)\bigr)}{N}}
\end{equation}
has probability at least $1-p$ and does not depend on $F$.
For any bounded Lipschitz $F$, define
$F_r:=\sum_{j=1}^J F(z_j)\mathbf{1}_{I_j}$.
Since $I_j\subseteq B_j$, we have $\|F-F_r\|_\infty\le rL_F$. Hence
\begin{equation}\label{equ.proofofmain}
\begin{aligned}
&\left\|\int_{\mathcal{P}}F\,d\nu-\frac1N\sum_{i=1}^N F(X_i)\right\|_{\ell^\infty}\\
&\quad\le 2rL_F+
\left\|\sum_{j=1}^J F(z_j)
\left(\nu(I_j)-\frac{S_j}{N}\right)\right\|_{\ell^\infty}\\
&\quad\le 2rL_F+\|F\|_\infty
\sum_{j=1}^J\left|\nu(I_j)-\frac{S_j}{N}\right|.
\end{aligned}
\end{equation}
Combining \cref{equ.proofmain,equ.proofofmain} and $J=\kappa(r)$ proves the assertion simultaneously for all $F$ on the same event.
\end{proof}
\subsection{A generalization theorem for MPNNs}
We use the classification setting and norm conventions of \cref{theorem:generalization}, with $\mathcal{P}=\X\times\{0,1\}^C$. When $d'=C$, cross-entropy composed with softmax satisfies the required Lipschitz condition on the bounded prediction range, with a constant that may depend on $C$ and $B_{\mathcal{H}}$.
\begin{lemma}[\cite{DIDMs25}, Lemma 85.]\label{lemma:lipinfbound_for_gen}
Under the above classification assumptions, let
\[
f\in{\rm Lip}(\X,\R^{d'},L_{\mathcal{H}},B_{\mathcal{H}})
\]
and put $h_f(x,c):=\mathcal{E}(f(x),c)$. Then
\[
\|h_f\|_\infty\le L_{\mathcal{E}}(B_{\mathcal{H}}+1)+|\mathcal{E}(0,0)|,
\]
and $L_{\mathcal{E}}\max\{L_{\mathcal{H}},1\}$ is a Lipschitz constant of $h_f$ with respect to $\rho$.
\end{lemma}
\begin{proof}
Since $\|c\|_{\ell^\infty}\le1$ and $\|f(x)\|_{\ell^\infty}\le B_{\mathcal{H}}$,
\[
\begin{aligned}
|h_f(x,c)|&\le|\mathcal{E}(0,0)|
+L_{\mathcal{E}}\bigl(\|f(x)\|_{\ell^\infty}+\|c\|_{\ell^\infty}\bigr)
\\
&\le|\mathcal{E}(0,0)|+L_{\mathcal{E}}(B_{\mathcal{H}}+1).
\end{aligned}
\]
Moreover,
\[
\begin{aligned}
|h_f(x,c)-h_f(x',c')|
&\le L_{\mathcal{E}}\bigl(L_{\mathcal{H}}\delta_\square(x,x')
+\|c-c'\|_{\ell^\infty}\bigr)\\
&\le L_{\mathcal{E}}\max\{L_{\mathcal{H}},1\}
\rho((x,c),(x',c')).
\end{aligned}
\]
\end{proof}
By \cref{theorem:Lip}, the class $\mathcal{H}$ of $T$-layer MPNNs with readout and the fixed dimensions and parameter bounds specified there satisfies
\[
\mathcal{H}\subseteq{\rm Lip}(\X,\R^{d'},L_{\mathcal{H}},B_{\mathcal{H}}).
\]
Thus a generalization bound for $\mathrm{Lip}$ also applies to $\mathcal{H}$.

\MPNNsGeneralizationTheorem*
\begin{proof}
Equip $\mathcal{P}=\X\times\{0,1\}^C$ with the sum metric $\rho$ specified in the classification setting. Choose a radius-$\epsilon$ cover of $\X$ with $\kappa_{\X}(\epsilon)$ centers and pair each center with every label in $\{0,1\}^C$. These pairs form a radius-$\epsilon$ cover of $\mathcal{P}$, since
\[
\rho((x,c),(z,c))=\delta_\square(x,z).
\]
Thus the covering number of $\mathcal{P}$ satisfies
\[
\kappa_{\mathcal{P}}(\epsilon)\le2^C\kappa_{\X}(\epsilon).
\]
For $\Upsilon\in\mathrm{Lip}$, set $F_\Upsilon(x,c):=\mathcal{E}(\Upsilon(x),c)$. By \cref{lemma:lipinfbound_for_gen}, each $F_\Upsilon$ is $L$-Lipschitz and bounded in absolute value by $\mathrm{B}$.
Apply \cref{theorem:robustGen} to $\mathcal{P}$. With probability at least $1-p$, on a measurable event independent of $\Upsilon$, simultaneously for all such functions,
\begin{equation}\label{inequality:generalizationProof}
\begin{aligned}
&\left|\int_{\mathcal{P}}F_\Upsilon\,d\nu
-\frac1N\sum_{i=1}^N F_\Upsilon(X_i,C_i)\right|\\
&\quad\le
2\epsilon L+{\rm B}
\sqrt{\frac{2\bigl(\kappa_{\mathcal{P}}(\epsilon)\ln2+\ln(1/p)\bigr)}{N}}.
\end{aligned}
\end{equation}
Substituting the product covering bound and the definitions of statistical and empirical risk proves the result. Since the same event works for every $\Upsilon$ in the class, it also works for any sample-dependent choice $\Upsilon_{\mathbf{X}}$.
\end{proof}
\section{Stability of MPNNs to graph subsampling}\label{AP:stab}
We follow the proof of \cite[Theorem 4.3]{signal23} to prove \cref{thm:MPNN_samp}.
\theoremstability*
\begin{proof}
    By the Lipschitz assumption on the joint map, 
    \[\delta_{\square}\big(\Sigma, \Sigma(\Lambda)\big) \leq L \delta_{\square}\Big( \big(W,f\big), \big(\mathbb{G}(W,\Lambda),f(\Lambda)\big)\Big).\]
    Hence, 
    \[\mathbb{E}\Big(\delta_{\square}\big(\Sigma, \Sigma(\Lambda)\big)\Big) \leq L \mathbb{E}\bigg(\delta_{\square}\Big( \big(W,f\big), \big(\mathbb{G}(W,\Lambda),f(\Lambda)\big)\Big)\bigg),\]
    and the claim of the theorem follows from \cref{lem:second-sampling-garphon-signal02}.
\end{proof}
\section{Notations}
\label{notations}

$[n]=\{1,\ldots,n\}$ (\cpageref{p:set}).

$G=\{V,E\}$: graph with nodes $V$ and edges $E$ (\cpageref{def:graphon}).

$(W,f)$: graphon-signal/kernel-signal (\cpageref{def:graphon}).

$\mathcal{W}_0$: space of graphons (\cpageref{def:graphon}).

$\cD_0$: space of directed graphons (\cpageref{def:graphon}).

$\cS_1$: space of symmetric kernels (\cpageref{def:graphon}).

$\cW_1$: space of kernels (\cpageref{def:graphon}).

$\cL^{\infty}_r([0,1];\R^d)$: signal space (\cpageref{n:Linfr1}).

$\WLrd$: the undirected graphon-signal space (\cpageref{n:Linfr1}).

$\DLrd$: the directed graphon-signal space (\cpageref{n:Linfr1}).

$\SLrd$: the symmetric kernel-signal space (\cpageref{n:Linfr1}).

$\KLrd$: the general kernel-signal space (\cpageref{n:Linfr1}).

$\mu$: standard Lebesgue measure on $[0,1]$ (\cpageref{p:partition}).

$\mathcal{P}_k=\{P_1,\ldots,P_k\}$: partition (\cpageref{p:partition})

$W_G$: induced graphon from the graph $G$ (\cpageref{def:induced-graphon}).

$(W,f)_{(G,\mathbf{f})} = (W_G,f_{\mathbf{f}})$: induced graphon-signal (\cpageref{def:induced-graphon}).

$\Phi(x,y)$: message function (\cpageref{messagefunction}).

$\xi_{k,\text{rec}},\xi_{k,\text{trans}}: \RR^d\rightarrow\RR^p$: receiver and transmitter message functions (\cpageref{messagefunction}).

$\Phi_f:[0,1]^2\rightarrow\RR^p$: message kernel (\cpageref{p:message}).

$\mathrm{Agg}(W,Q)$: aggregation of message $Q$ with respect to graphon $W$ (\cpageref{p:message}).

$\mu^{(t)}$: The update function at the \(t\)th layer of an MPNN (\cpageref{definition:altGraphFeat}).

$\mathbf{f}^{(t)}$: The signal of a graph-signal after applying \(t\) message passing and update layers (\cpageref{definition:altGraphFeat}).

$f^{(t)}$: The signal of a graphon-signal after applying \(t\) message passing and update layers (\cpageref{definition:altGraphonFeat}).

$\Theta_t(W,f)$: the output of the MPNN applied on $(W,f)\in \WLrd$ at layer $t\in[T]$ (\cpageref{p:mpnnt}).

$\Theta(W,f)$: the output of the MPNN without readout applied on $(W,f)\in \WLrd$ at layer $T$ (\cpageref{p:mpnnt}).

$\Theta(W,f)$: the output of the MPNN applied on $(W,f)\in \WLrd$ at layer $T$ after readout (\cpageref{p:mpnnt}).

$\norm{\cdot}_{\square}$: the cut norm (\cpageref{eq:3}).

$d_{\square}(W,V)$: the cut metric (\cpageref{eq:4}).

$\delta_{\square}(W,V)$: the cut distance (\cpageref{eq:4}).

$\norm{\cdot}_{\square_{\times}}$: the product cut norm (\cpageref{eq:4}).

$S'_{[0,1]}$: the set of measurable measure preserving bijections between co-null sets of $[0,1]$ (\cpageref{p:nullset}).

$\delta_{\square}\big((W,f),(V,g)\big)$: graphon-signal cut distance (\cpageref{eq:gs-metric}).

$\DLrdt$: directed graphon-signal space modulo zero cut distance (\cpageref{p:graphoncutspace}).

$\SLrdt$: symmetric kernel-signal space modulo zero cut distance (\cpageref{p:graphoncutspace}).

$\KLrdt$: general kernel-signal space modulo zero cut distance (\cpageref{p:graphoncutspace}).

$\WLrdt$: undirected graphon-signal space modulo zero cut distance (\cpageref{p:graphoncutspace}).

$\mathcal{S}^{p\to d}_{\mathcal{P}_k}$: the space of step functions of dimension $p$ to $\mathbb{R}^d$ over the partition $\mathcal{P}_k$ (\cpageref{def:step}).

$\cW_0\cap\mathcal{S}^2_{\mathcal{P}_k}$: the space of step undirected graphons (\cpageref{def:step}).

$\cD_0\cap\mathcal{S}^2_{\mathcal{P}_k}$: the space of directed step graphons (\cpageref{def:step}).

$\cS_1\cap\mathcal{S}^2_{\mathcal{P}_k}$: the space of symmetric step kernels (\cpageref{def:step}).

$\cW_1\cap\mathcal{S}^2_{\mathcal{P}_k}$: the space of general step kernels (\cpageref{def:step}).

$\mathcal{L}_r^{\infty}[0,1]\cap\mathcal{S}^1_{\mathcal{P}_k}$: the space of step signals (\cpageref{def:step}).

$[\WLrd]_{\mathcal{P}_k}$: the space of graphon-signal stochastic block models with respect to the partition $\mathcal{P}_k$ (\cpageref{def:step}).

$[\DLrd]_{\mathcal{P}_k}$: the space of directional graphon-signal stochastic block models with respect to the partition $\mathcal{P}_k$ (\cpageref{def:step}).

$[\SLrd]_{\mathcal{P}_k}$: the space of Symmetric kernel-signal stochastic block models with respect to the partition $\mathcal{P}_k$ (\cpageref{def:step}).

$[\KLrd]_{\mathcal{P}_k}$: the space of kernel-signal stochastic block models with respect to the partition $\mathcal{P}_k$ (\cpageref{def:step}).

$W(\Lambda)$: random weighted graph (\cpageref{lem:second-sampling-garphon-signal00}).

$f(\Lambda)$: random sampled signal (\cpageref{lem:second-sampling-garphon-signal00}).

$\mathbb{G}(W,\Lambda)$: random simple graph (\cpageref{lem:second-sampling-garphon-signal00}).

${\rm Lip}(\X,R^{d'},L_{\mathcal{H}},B_{\mathcal{H}})$ (alternatively ${\rm Lip}$): be the set of all continuous functions $f:\X\mapsto\R^{d'}$, with a Lipschitz constants $L_{\mathcal{H}}$, that are bounded by $B_{\mathcal{H}}$ as in \cref{theorem:Lip} (\cpageref{p:lip}).
\end{document}